\documentclass{article}
\usepackage[preprint]{neurips_2026}
\setcitestyle{round,semicolon}

\usepackage[utf8]{inputenc}
\usepackage[T1]{fontenc}
\usepackage{hyperref}
\usepackage{url}
\usepackage{booktabs}
\usepackage{amsfonts}
\usepackage{nicefrac}
\usepackage{microtype}
\usepackage{xcolor}
\usepackage{soul}
\usepackage{multirow}
\usepackage{makecell}
\usepackage{booktabs}

\usepackage{amsmath}
\usepackage{amssymb}
\usepackage{amsthm}
\usepackage{bbm}
\usepackage{mathtools}
\usepackage{hyperref}
\usepackage{cleveref}
\hypersetup{colorlinks=true,linkcolor=blue!60!black,citecolor=blue!60!black,urlcolor=blue!60!black}
\usepackage{subcaption}
\usepackage{wrapfig}
\usepackage{bm}
\usepackage{mathrsfs}
\usepackage{enumitem}
\usepackage{amsthm}

\usepackage{algorithm}
\usepackage{algpseudocode}
\usepackage{float}
\usepackage{tikz}
\usepackage{threeparttable}
\usetikzlibrary{positioning, arrows.meta, calc, fit, backgrounds}
\newtheorem{theorem}{Theorem}
\newtheorem{assumption}{Assumption}

\newtheorem{proposition}{Proposition}
\newtheorem{remark}{Remark}

\theoremstyle{definition}
\newtheorem{definition}{Definition}

\newcommand*{\eg}{\emph{e.g.}{}}
\newcommand*{\ie}{\emph{i.e.}{}}

\renewcommand{\d}{\mathrm{d}}

\newcommand{\R}{\mathbb{R}}
\newcommand{\calP}{\mathcal{P}}
\newcommand{\calS}{\mathcal{S}}
\newcommand{\scrX}{\mathscr{X}}
\newcommand{\scrS}{\mathscr{S}}
\newcommand{\scrA}{\mathscr{A}}

\allowdisplaybreaks

\definecolor{mygreen}{RGB}{78, 150, 73}
\definecolor{myred}{RGB}{180, 78, 84}

\title{Learning to Test: Latent Representation\\ Learning for Dynamic Instability Detection}

\author{
  Minxing Zheng\\
  Carnegie Mellon University\\
  \And
  Zewei Deng\\
  University of Minnesota\\
  \AND
  Liyan Xie\\
  University of Minnesota\\
  \And
  Shixiang Zhu\\
  Carnegie Mellon University\\
}

\begin{document}

\maketitle

\begin{abstract}
Stability assessment is central to safety-critical systems governed by differential–algebraic equations (DAEs), whose dynamics are constrained by physical laws and admissibility conditions. These systems operate under stochastically varying environmental inputs, so stability must be reassessed as the context distribution shifts, yet repeated large-scale DAE simulation is prohibitive in high-dimensional or real-time settings. We propose a test-oriented learning framework that tests a trajectory-level chance constraint from deployment contexts alone. Rather than re-estimating physical parameters or re-solving the DAE, we learn a latent representation from baseline context--trajectory data that retains stability-relevant information and is regularized toward a tractable reference distribution, and calibrate a latent discrepancy test over the composite safety null. We establish exact validity for oracle calibration and asymptotic validity for its empirical implementation under class-conditional latent invariance, and show that the guarantee degrades gracefully when this invariance holds only approximately. On simulated dynamical systems and a power-system real-world example, our approach strikes a balance between validity and power, whereas other baseline methods each fail in at least one regime. An ablation that injects controlled violations of the invariance assumption shows that decisions stay correct while the violation is below the safety margin, as the theory predicts.
\end{abstract}

\section{Introduction}

Many safety-critical scientific and engineering systems evolve according to physical laws while satisfying hard algebraic constraints, including power grids, structural systems, and biochemical networks \citep{kunkel2006differential,kundur2007power,ascher1998computer,rao2003stochastic}. Differential--algebraic equations (DAEs) provide a natural representation by coupling differential dynamics with algebraic constraints that define a feasible manifold \citep{brenan1995numerical}. In operation, however, exogenous contexts such as loads, environmental conditions, and external forces vary stochastically \citep{arnold2006random,khasminskii2011stochastic}, shifting the operating point and altering the system's dynamics and stability properties \citep{lee20136,milano2010power,roberts2003random,sobczyk2013stochastic,chiang1987foundations}. Stability must therefore be reassessed as the context distribution evolves \citep{liberzon2003switching}.

Yet repeated stability assessment faces two fundamental challenges: ($i$) From a computational standpoint, each environmental change may require re-solving high-dimensional and often stiff DAE systems or recomputing local linearizations and spectral properties---tasks that quickly become prohibitive in large-scale or real-time settings. ($ii$) From a statistical standpoint, observations are typically available only from a single baseline domain, corresponding to a \emph{singleton} within the safe regime. Such data capture system behavior under limited operating conditions and do not span the full set of stable scenarios. As a result, conventional end-to-end learning approaches that map inputs directly to stability labels can be systematically biased: they perform well near the training distribution but fail under distribution shift, while making limited use of the richer trajectory information available in the baseline data.

These challenges motivate a central question: how can we assess the stability and safety of constrained dynamical systems under stochastically varying inputs without incurring the cost of repeated full-scale analysis, while also ensuring reliable generalization beyond observed conditions? More broadly, \emph{how can we efficiently detect, with uncertainty quantification, when such a system is approaching or has entered an unsafe regime in a changing environment?} In this work, we propose a \emph{test-oriented learning} framework, inspired by decision-focused learning \citep{mandi2024decision} and tailored to stability hypothesis testing. Using context--trajectory observations from a safe baseline domain, we learn a latent representation that preserves stability-relevant dynamics and enables calibrated distributional testing using only contextual variables at deployment. Rather than optimizing for trajectory prediction or reducing stability to pointwise labels, the representation is designed for the downstream inferential objective: detecting whether a shifted context distribution violates the prescribed stability requirement.

\begin{figure}[t]
    
    \centering
    \includegraphics[width=.95\linewidth]{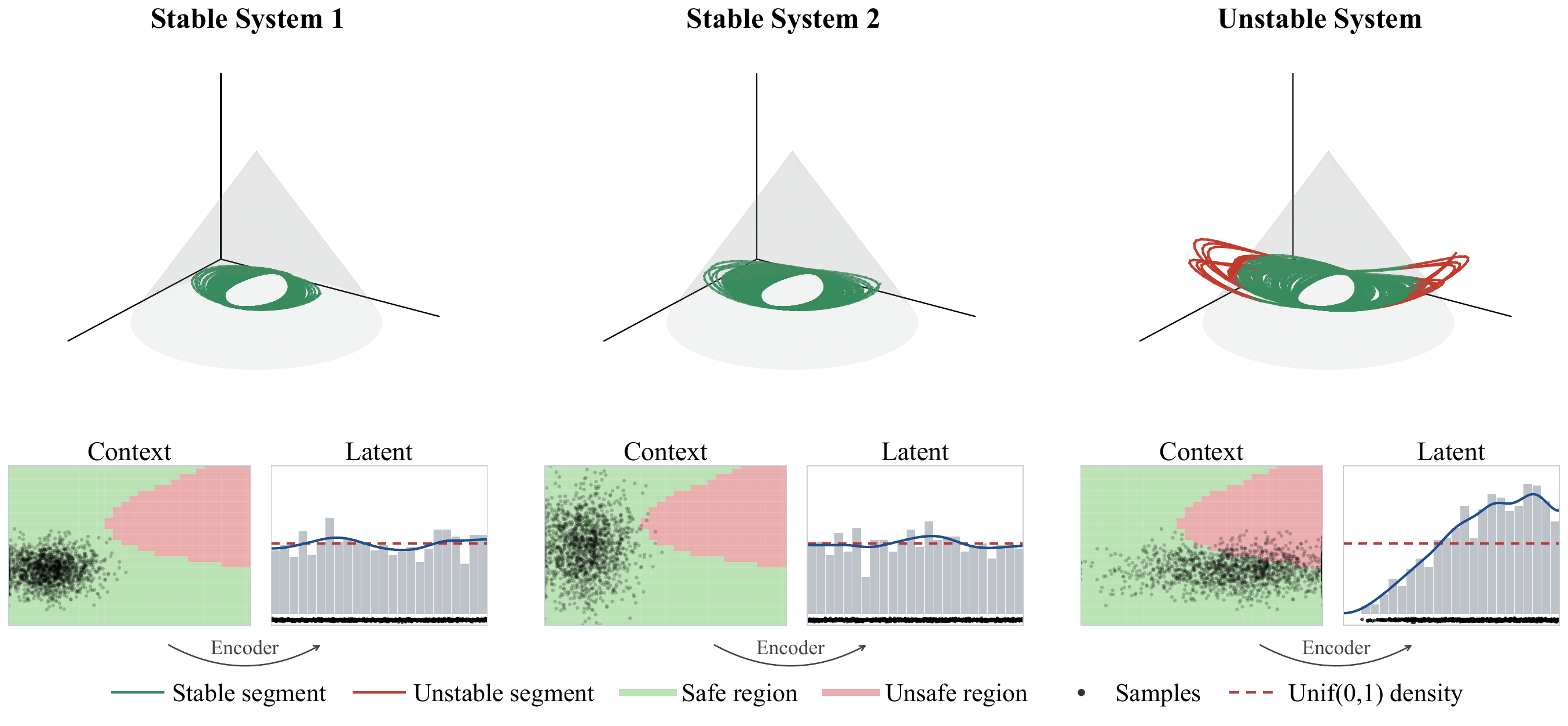}
    \caption{Illustration of Cartesian pendulum systems and the one-dimensional latent representations induced by our method. The system maps a context variable (\eg, initial conditions or external disturbances) to a pendulum trajectory.
    A system is considered \emph{stable} when the context follows the ``safe'' distribution.
    The top row shows representative trajectories under three exogenous context distributions: green segments satisfy the angular stability criterion, while red segments indicate violations. Under stable regimes, the learned latent representation is approximately uniform; under unstable regimes, it deviates noticeably from uniformity.}
    \label{fig:stable_unstable_latent}
    
\end{figure}

As illustrated by Figure~\ref{fig:stable_unstable_latent}, our central idea is to learn a latent representation of contextual variables that encodes stability-relevant features of the constrained dynamics, while being regularized toward a simple and tractable reference distribution. Trained on baseline data from a certified safe operating regime, we learn a structured latent geometry for subsequent distributional testing.
At deployment, safety monitoring reduces to a latent space distributional hypothesis test. If contexts under the new regime remain close to the reference law, stability is inferred to persist; systematic deviations signal elevated instability risk. By learning the representation end-to-end with respect to this testing objective, we bypass repeated DAE simulation, extract stability-relevant structure from baseline trajectories, and enable scalable context-only monitoring under stochastic environmental change.

Our work makes three main contributions: ($i$) We introduce a \emph{\ul{test-oriented learning}} framework that directly targets stability verification under distribution shift, bypassing explicit system identification and repeated trajectory simulation at deployment. ($ii$) We provide a \emph{\ul{latent reformulation}} of stability hypothesis testing, showing that, under explicit conditions, a trajectory-level chance constraint in constrained dynamical systems can be equivalently characterized as a distributional test in latent space, and establish asymptotic validity and consistency for the resulting test. ($iii$) We conduct an \emph{\ul{extensive empirical evaluation}} on both synthetic dynamical systems and a real-world transient-stability assessment benchmark based on the IEEE 39-bus power system \citep{Sokolovic2024TSA}, demonstrating reliable and powerful detection of instability under context shift. Together, these contributions establish a scalable and statistically principled framework for context-only stability assessment in stochastic constrained dynamical systems.

\paragraph{Related Work.}

Stability analysis studies how dynamical systems respond to perturbations and whether trajectories remain near equilibria or reference motions. Classical theory uses linearization, Lyapunov functions, and invariance principles \citep{lyapunov1992general,lasalle1960some}, with broad applications in power systems \citep{anderson2008power,kundur2007power,kundur2004definition,sauer2017power,wu2023transient}, control \citep{venkatesh2024enhancing,walsh2002stability,wang2010analysis}, transportation \citep{bastin2007lyapunov,kesting2008reaction,lovisari2014stability,sadri2013stability,wang2020stability}, chemical processes \citep{favache2009thermodynamics,wang2011analysis}, and epidemiology \citep{bonzi2011stability,cangiotti2024survey,earn2025global,fall2007epidemiological,korobeinikov2002lyapunov}. These methods commonly assess stability pointwise under fixed operating conditions or deterministic perturbations \citep{abbasi2024switched,chiang1987foundations,lovisari2014stability}, while stochastic extensions typically model process noise \citep{li2022analysis}. We instead treat exogenous contexts as random and test whether their distribution satisfies a trajectory-level chance constraint after a shift. Related probabilistic transient-stability methods sample operating scenarios \citep{papadopoulos2016probabilistic}, but do not address context-only testing when deployment trajectories are unavailable.

Scientific machine learning learns dynamical surrogates using recurrent models \citep{elman1990finding,rumelhart1986learning,hochreiter1997long,cho2014properties,chung2014empirical,yuchi2025new}, neural and latent ODEs \citep{bilovs2021neural,chen2018neural,chen2025global,li2023neural,rubanova2019latent}, and physics-informed or DAE-aware architectures \citep{raissi2019physics,rackauckas2020universal,huang2024minn,koch2025learning}. These methods generally target prediction or system identification rather than testing under distribution shift. A direct alternative learns $X\mapsto y(X)$ and applies a binomial test. This is effective when label prediction transfers to $P_1$, but predictors trained on one baseline domain can be brittle under shift \citep{ovadia2019can,rosenfeld2022online,yu2024learning,yuan2022towards,zhang2024single}, while a hard label discards trajectory severity and temporal structure \citep{rubanova2019latent,xiao2022dynamic,zhang2024trajectory}. Our method instead uses trajectory reconstruction and stability-aware supervision to learn a representation for stability testing \citep{biggs2023mmd,kim2021classification,kirchler2020two,podkopaev2023sequential}.

Our test-oriented perspective is also related to learning-to-optimize and decision-focused learning \citep{zhu2022distributionally,elmachtoub2022smart,li2016learning,mandi2024decision,wang2025gen}, including differentiable optimization \citep{wilder2019end} and ranking-based approaches \citep{mandi2022decision}. It is also connected to learned statistical tests, including representation-based two-sample tests \citep{ginsberg2022learning,kirchler2020two,wang2022innovations} and sequential procedures \citep{wei2021goodness,zhang2025recurrent,zhou2025sequential,zhou2026score}. Unlike the former, our downstream object is a hypothesis test rather than an optimization decision; unlike generic learned testing, our representation is trained using observed trajectories and a stability-aware reconstruction loss, and the inferential target is a stability chance constraint.

\section{Problem Setup}
\label{sec:problem_setup}

We consider a dynamical system, parameterized by deterministic but {\it unknown} physical parameters $\theta \in \Theta$, with state variables $\{s(t)\}_{t \ge 0}$ that evolves on the time horizon $t\in[0,+\infty]$ according to the following differential--algebraic equation (DAE):
\begin{subequations}\label{eq:dae_main}
\begin{align}
\d s(t) / \d t &= f_\theta\big(s(t), a(t); X\big), t \ge 0; ~ \text{ and}~ s(0) = s_0(\theta, X), \label{eq:dae_main_ode}\\
0 &= g_\theta\big(s(t), a(t); X\big).
\label{eq:dae_main_alg}
\end{align}
\end{subequations}
Here $X \in \scrX$ is a {\it random} exogenous context (fixed once realized) drawn from a distribution,   $a(\cdot)\in \scrA$ is an auxiliary algebraic variable enforcing instantaneous constraints, and $s_0(\cdot,\cdot)$ specifies the initial state of the system at time $t=0$.
For notational clarity, given a context realization $X=x$, we write the resulting state trajectory as $\{s_\theta(t;x)\}_{t\ge 0}\in \scrS$, where $\scrS$ denotes the space of admissible state trajectories, and the associated algebraic variables as $\{a_\theta(t;x)\}_{t\ge 0}$.
We assume that $f_\theta$ and $g_\theta$ are unknown, continuously differentiable in $(s,a)$ and that \eqref{eq:dae_main} is well-posed, so that for each $(\theta,x)$ there exists a unique solution $(s_\theta(\cdot;x),a_\theta(\cdot;x))$.
When $g_\theta \equiv 0$ in \eqref{eq:dae_main_alg}, the system in \eqref{eq:dae_main} reduces to an ordinary differential equation (ODE).
It is worth emphasizing that the system trajectory $\{s_\theta(t;X)\}$ is stochastic, with randomness induced by the distribution of $X$. Concrete examples are given in Appendix~\ref{app:motivating-exps}.

Let $\eta:\scrS\to\R$ be a user-specified trajectory functional that
measures departure from an admissible operating region. For example, $\eta(s(\cdot))=\max_{t\in[0,T]} |s(t)|$,
records the maximum state magnitude over the time horizon $[0,T]$.
Given a tolerance threshold $\tau>0$, a context $x$ is said to produce
a trajectory violation if
$\eta(s_\theta(\cdot;x))>\tau$. We define the trajectory-level violation indicator as
\[
y(x)
\coloneqq
\mathbbm{1}
\left\{
\eta\bigl(s_\theta(\cdot;x)\bigr)>\tau
\right\}.
\]
In this work, we consider system stability at the distributional level by controlling the probability of this trajectory-level violation under $X\sim P$.
\begin{definition}[$(\tau,\alpha)$-Stability]\label{def:stab}
The system is $(\tau,\alpha)$-\emph{stable} under $P$ if
\begin{equation}\label{eq:stab}
\mathbb E_{X\sim P}[y(X)]
=\mathbb P_{X\sim P}\!\left[\eta\bigl(s_\theta(\cdot;X)\bigr)>\tau\right]
\le\alpha,
\end{equation}
\end{definition}
\ie, the prescribed stability requirement
$\eta(s_\theta(\cdot;X))\le\tau$ holds with probability at least $1-\alpha$.
Here ``stability'' means satisfaction of a trajectory-level chance constraint over the chosen observation horizon; it does not assert asymptotic stability of individual trajectories.

\begin{definition}[Safe Regime]
\label{def:safe_regime}
Let $\calP(\scrX)$ denote the set of Borel probability measures on $\scrX$. For fixed $(\theta,\tau,\alpha)$, the \emph{safe regime} is
\begin{equation}\label{eq:safe}
\calP_0\coloneqq
\left\{P\in\calP(\scrX):\mathbb P_{X\sim P}\!\left[\eta\bigl(s_\theta(\cdot;X)\bigr)>\tau\right]\le\alpha\right\},
\end{equation}
\ie, $\calP_0$ is the set of $(\tau,\alpha)$-stable distributions in the sense of Definition~\ref{def:stab}.
\end{definition}

The safe regime $\calP_0$ is unknown because the physical parameters $\theta$ and governing mappings $f_\theta$ and $g_\theta$ are unknown.
We assume access to observations collected from a baseline domain with distribution $P_0\in\calP_0$.  Specifically, we observe i.i.d. samples $\{(X_i,\calS_i)\}_{i=1}^n$, where $X_i\sim P_0$ and $\calS_i=\{s_\theta(t;X_i)\}_{t\in[0,T]}.$
Conditional on $X_i$, the trajectory $\calS_i$ is deterministic and its randomness is induced by $X_i\sim P_0$. These data characterize a safe operating regime without revealing the underlying parameters or dynamics.

We now consider a new domain in which contexts $\{X_j\}_{j=n+1}^{n+m}$ are drawn i.i.d.\ from a potentially new distribution $P_1$, but the corresponding system trajectories are unobserved.
The key question is whether the stability verified in the baseline domain persists under this distribution shift.
Determining whether $P_1$ belongs to the safe regime naturally leads to the following hypothesis test
\begin{equation}
\label{eq:hypothesis}
H_0:~P_1\in\calP_0
\qquad
\text{vs.}
\qquad
H_1:~P_1\notin\calP_0.
\end{equation}
Rejecting $H_0$ provides statistical evidence that the prescribed stability requirement may no longer hold under $P_1$, thereby triggering an alarm for further diagnosis or intervention.

\section{Proposed Framework: Learning to Test}

\begin{wrapfigure}{r}{0.6\textwidth}
    
    \centering
    \includegraphics[width=\linewidth]{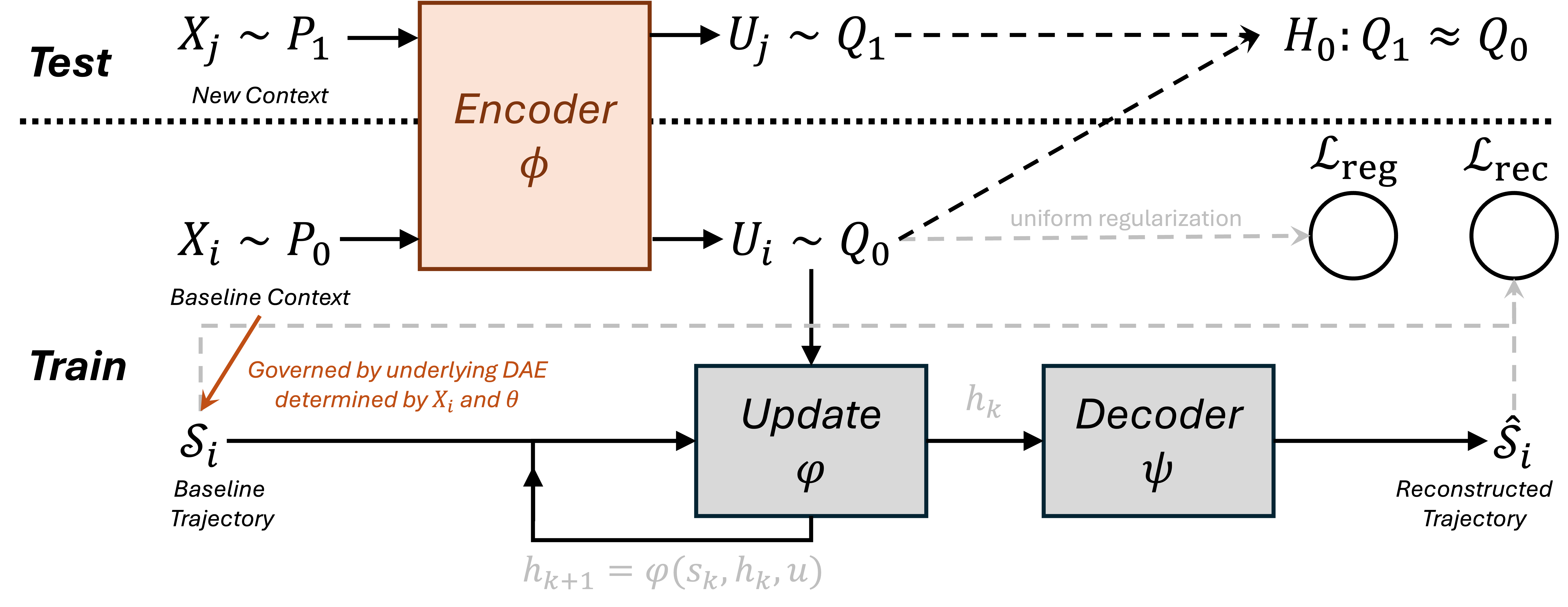}
    
    \caption{Overview of the ``Learning-to-Test'' paradigm.
    }
    \label{fig:architecture}
    
\end{wrapfigure}

We propose an end-to-end framework, termed ``learning to test'' (\texttt{L2T}), that directly targets the hypothesis testing problem in \eqref{eq:hypothesis}, rather than first estimating the physical parameters and then simulating state trajectories under the new domain.
We learn a latent representation of the contextual variable using baseline data so that it captures distributional factors governing stability of the DAE system. Uniformity regularization provides a bounded and interpretable reference geometry. At deployment, we measure the latent discrepancy from this reference and calibrate its rejection threshold over the composite safe regime. Thus, calibration supplies validity within the shift family specified below, while representation learning determines the power of the resulting test. Figure~\ref{fig:architecture} summarizes the proposed \texttt{L2T} paradigm.

This approach offers two advantages: ($i$) Representation learning can be carried out efficiently {\it offline} using baseline data, without requiring repeated simulation in the new domain. ($ii$) By operating directly at the level of the testing objective, the method bypasses the need to estimate physical parameters or repeatedly solve the DAE, focusing instead on detecting shifts relevant to stability.

\subsection{Representation Learning for Stability Testing}
\label{sec:representation}

We use a learned history-state surrogate to model the observed trajectories. A continuous-time formulation motivates the update below, while our implementation uses a discrete-time recurrent architecture.
Specifically, for a continuous-time series $\{s(t),t\geq 0\}$, we define a compact \emph{history embedding} $h(t) \in \mathcal{H}$,
where $\mathcal{H}$ denotes the embedding space. This history embedding encapsulates the key information about the past observations up to time $t$. Conceptually, its evolution can be described by
\begin{equation}
    \label{eq:ode_h}
    \d h(t) / \d t
    = \varphi(s(t), h(t); \theta, x),
\end{equation}
where $\varphi$ is a learned update function conditioned on the observed state and context; we do not impose the governing DAE equations as constraints on this surrogate.
An Euler discretization of this conceptual model would take the form
\begin{equation}\label{eq:ode-euler}
h_{k+1}=h_k+\varphi(s_k, h_k; \theta, x)\Delta t_k, \quad k=1,\ldots,K-1,
\end{equation}
where $h_k$ and $h_{k+1}$ are the history embeddings at times $t_k$ and $t_{k+1}$, respectively, and $\Delta t_k=t_{k+1}-t_k$ is time difference between them.
Our experiments instead use the recurrent update described below.

\paragraph{Latent Representation.} We seek to construct a latent representation of the contextual variable $X$ such that: $(i)$ it retains information about the observed trajectories relevant to stability; and $(ii)$ it is regularized toward a uniform law.
To this end, we introduce an \emph{encoder}
\[
\mathcal \phi:\scrX \to [0, 1]^d,
\qquad
u=\phi(x),
\]
which maps each context $x$ to a latent variable $u$.
Accordingly, we reparameterize the history-embedding update and state reconstruction as
\begin{equation}
\label{eq:h-embedding}
h_{k+1}=\varphi(s_k,h_k,u),\qquad \hat s_k=\psi(h_k).
\end{equation}
Here $\varphi$ denotes the history-update function using the current observation $s_k$, previous history $h_k$, and latent representation $u$, corresponding by a slight abuse of notation to the right-hand side of \eqref{eq:ode-euler}; $\psi$ is a \emph{decoder} that maps the history embedding $h_k$ back to the state space.
The history-update function $\varphi$ can typically be implemented via a recurrent architecture (\eg, recurrent neural network, long-short term memory \citep{yu2019review}, or Transformer \citep{vaswani2017attention}).

\paragraph{Training Objectives.}

The encoder $\phi$ is trained jointly with decoder $\psi$ and update function $\varphi$ to construct latent representations that preserve stability-relevant information while remaining distributionally uniform.
Given the baseline data $\{(x_i,\mathcal S_i)\}_{i=1}^n$, where $\mathcal S_i=\{s_{i,k}\}_{k=1}^{K_i}$ consists of $K_i$ observations and is associated with context $x_i$, the training objective consists of three components:

\begin{enumerate}[leftmargin=*, itemsep=0pt, topsep=0pt, parsep=0pt, partopsep=0pt]
    \item \emph{Reconstruction loss}: To encourage the latent representation to retain information about trajectories observed in the baseline regime, we minimize the reconstruction loss:
    \begin{equation}
        \label{eq:loss-reconstruction}
        \mathcal L_{\mathrm{rec}}(\phi, \psi, \varphi) \coloneqq \frac{1}{n}\sum_{i=1}^n \frac{1}{K_i}\sum_{k=1}^{K_i} \|s_{i,k}-\hat s_{i,k}\|^2,
    \end{equation}
    where $\hat s_{i,k}$ are the reconstructed output in \eqref{eq:h-embedding} with $u=\phi(x_i)$.
    \item \emph{Latent uniformity regularization}: To facilitate distributional testing in the latent space, we regularize the learned representations toward a uniform reference distribution by minimizing
    \begin{equation}
    \label{eq:loss-uniform}
    \mathcal L_{\mathrm{reg}}(\phi)
    \coloneqq
    \mathrm{KL}\left[p_U ~\middle\|~ \mathrm{Unif}(0,1)^d\right] = - \mathcal H(p_U) + \text{const},
    \end{equation}
    where $\mathrm{KL}[\cdot \| \cdot]$ represents Kullback–Leibler (KL) divergence, $p_U$ denotes the empirical distribution of the latent variables $\{u_i\}_{i=1}^n$ with $u_i=\phi(x_i)$, and $\mathcal H(p)= -\int_{[0,1]^d} p(u)\log p(u)~\d u$ is the differential entropy.
    Minimizing this objective is equivalent to maximizing the entropy of the latent representation, thereby encouraging uniformity. This regularization supplies the fixed reference geometry used by the calibrated test in Section~\ref{sec:latent-test}.
    \item \emph{Stability-aware reweighting}:
    To preserve sensitivity to the stability condition, we augment the reconstruction loss in \eqref{eq:loss-reconstruction} with an asymmetric penalty based on the stability metric. For example, using $\eta(s(\cdot)) = \max_{t\in[0,T]} |s(t)|$, the sample-level weight can be defined as
    \begin{align*}
    \omega_{\mathrm{stb}}(s,\hat s)
    =
    \begin{cases}
    \omega, & |s|>\tau \text{ and } \operatorname{sign}(s)\hat s < \operatorname{sign}(s)s,\\
    1, & \text{otherwise},
    \end{cases}
    \end{align*}
    for some $\omega>1$, where $\operatorname{sign}(s)=1$ if $s>0$ and $\operatorname{sign}(s)=-1$ if $s<0$.
    This weighting scheme distinguishes errors based on whether the trajectory lies in the unstable region and on the direction of the reconstruction error. When $|s|\le\tau$, the trajectory is considered stable and assigned a unit weight. When $s>\tau$ or $s<-\tau$, the weight depends on whether the reconstruction exaggerates or attenuates the instability. In particular, the model is encouraged to overestimate positive instability and underestimate negative instability. This asymmetric design intentionally exaggerates truly unstable trajectories, improving sensitivity near the stability boundary.
\end{enumerate}

The overall objective is therefore
\begin{equation}
    \label{eq:loss_total}
    \mathcal L(\phi, \psi, \varphi) \coloneqq
    \frac{1}{n}\sum_{i=1}^n \frac{1}{K_i}\sum_{k=1}^{K_i} \omega_\mathrm{stb}(s_{i,k}, \hat{s}_{i,k}) \cdot \|s_{i,k}-\hat s_{i,k}\|^2 + \lambda \cdot \mathcal L_{\mathrm{reg}}(\phi),
\end{equation}
where $\lambda \ge 0$ controls the trade-off between reconstruction fidelity and latent uniformity regularization.

In practice, we represent $\phi$, $\psi$, and $\varphi$ using deep neural networks. We minimize the empirical objective $\mathcal L$ using stochastic gradient descent (SGD) or adaptive variants (\eg, Adam \citep{kingma2014adam}).
Each iteration samples a mini-batch from the baseline data, rolls out the latent dynamics via \eqref{eq:h-embedding} to obtain reconstructions, evaluates the three loss components, and updates all parameters by backpropagation.
We note that the latent uniformity term is implemented by estimating $p_U$ from the mini-batch latent representations and differentiating a tractable surrogate of \eqref{eq:loss-uniform} (\eg, via maximum mean discrepancy (MMD) \citep{gretton2012kernel} or a flow matching loss \citep{lipman2022flow, rezende2015variational}), so that gradients can be computed end-to-end. More implementation details are  in Appendix~\ref{append:learning-details}.

\subsection{Latent Distributional Test for Dynamic Instability Detection}
\label{sec:latent-test}

After training, we freeze the encoder $\phi$ and denote the latent representation variable as $U=\phi(X)$. For any context distribution $P$, denote the induced latent distribution by the encoder as
\[
Q_P \coloneqq \phi_{\#}P = P\circ\phi^{-1}.
\]
The training objective encourages $U$ to retain stability-relevant dynamical information while organizing the latent space around a simple reference distribution. We therefore seek to reformulate the stability test in \eqref{eq:hypothesis} as a distributional test on $Q_{P_1}$. To justify doing so, however, we must establish how the latent distribution reflects the probability of a stability violation.

We begin by separating the latent representations associated with stable and unstable trajectories. Because trajectories and their labels are available in the baseline domain under $P_0$, we define the baseline class-conditional latent distributions
\[
R_0 \coloneqq \operatorname{Law}_{P_0}\!\left(U\mid y(X)=0\right), \qquad R_1 \coloneqq \operatorname{Law}_{P_0}\!\left(U\mid y(X)=1\right).
\]
Here, $R_0$ and $R_1$ are the latent distributions associated with stable and unstable trajectories, respectively.
Without further structure on the deployment shift, the marginal latent distribution $Q_P$ need not identify the violation probability: changes in $Q_P$ may arise from changes within either stability class rather than from a change in the prevalence of unstable trajectories. We therefore assume that the class-conditional latent distributions remain invariant from the baseline domain to deployment, while the relative prevalence of stable and unstable trajectories may change.

\begin{assumption}[Class-Conditional Latent Invariance]
\label{ass:conditional-invariance}
For a deployment distribution $P_1$ under consideration and each class $a\in\{0,1\}$ with $P_1(y(X)=a)>0$, it holds that
\begin{equation}
\label{eq:conditional_latent_law}
\operatorname{Law}_{P_1}\!\left(U\mid y(X)=a\right)=R_a.
\end{equation}
\end{assumption}

\begin{remark}[A Weaker Invariance Requirement]
Assumption~\ref{ass:conditional-invariance} is a representation-space analogue of the label-shift assumption in domain adaptation \citep{lipton2018detecting,garg2020unified} and is closely related to generalized label shift \citep{tachetdescombes2020domain}. It is weaker than invariance of the class-conditional context distributions: the encoder can discard directions of $X$ that do not affect stability, and contexts may move along them without changing their latent law.
\end{remark}

Assumption \ref{ass:conditional-invariance} is crucial for the validity of our approach. Without a structural link between the baseline and deployment domains, the violation probability is not identified from the observed data: in Figure~\ref{fig:assumption}(a,b), two systems that agree on the baseline support produce the same observed data but have violation probabilities on opposite sides of $\alpha$ (Proposition~\ref{prop:nonident}). Training does not enforce the assumption, but exact invariance is not needed: if the class-conditional latent laws drift by at most $\varepsilon$ in $D$, inflating the calibrated threshold by $\varepsilon$ retains asymptotic validity (Proposition~\ref{prop:conditional-drift}; Figure~\ref{fig:assumption}c). Appendix~\ref{app:synthetic} checks the assumption empirically (Figure~\ref{fig:invariance_check}). For any distribution $P$ satisfying Assumption~\ref{ass:conditional-invariance}, the latent distribution can be expressed as a mixture of the class-conditional latent distributions: $Q_P=\mathbb P_P\!\left(y(X)=0\right)R_0+\mathbb P_P\!\left(y(X)=1\right)R_1.$
The violation probability is therefore exactly the mixing weight assigned to $R_1$. Consequently, the original stability condition $\mathbb P_P(y(X)=1)\le\alpha$ becomes a constraint on $Q_P$, providing the desired bridge from the trajectory-level hypothesis in \eqref{eq:hypothesis} to a distributional test in latent space. We thus obtain the following latent-space reformulation of the original test.

\begin{figure}[!t]
    \centering
    \includegraphics[width=\linewidth]{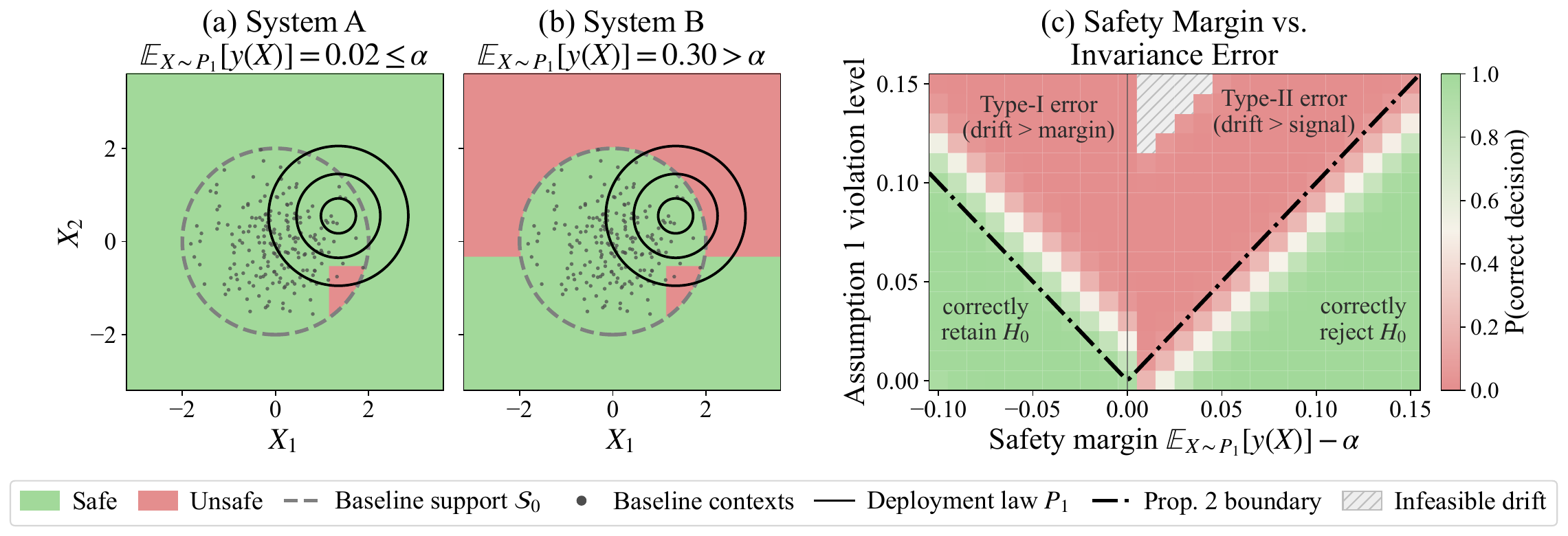}
    
    \caption{Why Assumption~\ref{ass:conditional-invariance} is needed and what violating it costs. Panels (a, b) show two systems that agree on the baseline support $\mathcal S_0$ and produce identical observed data, yet have violation probabilities on opposite sides of $\alpha=0.1$ (Proposition~\ref{prop:nonident}). Panel (c) shows the probability that the calibrated test decides correctly on the spring \texttt{L2T} latent under class-conditional drift toward the other class; the dash-dot boundary is predicted by Proposition~\ref{prop:conditional-drift}, not fitted. Hatched cells mark infeasible drift.}

    \label{fig:assumption}
    
\end{figure}

\begin{theorem}[Latent Reformulation]
\label{thm:reformulation}
Suppose Assumption~\ref{ass:conditional-invariance} holds and $R_0\neq R_1$, and let $Q^*\coloneqq\mathrm{Unif}([0,1]^d)$. Under a mild regularity condition (Appendix~\ref{app:thm-proof}), there exists a measurable transformation $T$ such that, after replacing $\phi$ by $\widetilde{\phi}=T\circ\phi$, for any separating integral probability metric $D$ there exists a threshold $r_\alpha>0$ for which the original hypothesis test in \eqref{eq:hypothesis} is equivalent to
\begin{equation}
\label{eq:reform-test}
H_0:\ D(Q_{P_1},Q^*)\le r_\alpha
\qquad\text{vs.}\qquad
H_1:\ D(Q_{P_1},Q^*)>r_\alpha,
\end{equation}
where $r_\alpha$ is determined by $\alpha$, $D$, and the transformed class-conditional latent distributions.
\end{theorem}

\begin{remark}
\label{rem:insight}
The theorem reformulates the original hypothesis test as a latent-space test against the known reference distribution $Q^*$, using a fixed threshold $r_\alpha$. For $D=\operatorname{TV}$ (total variation), the threshold is $r_\alpha=(\alpha/2)\operatorname{TV}(R_0,R_1)$; for $D=\operatorname{MMD}_k$ with a characteristic kernel $k$, it is $r_\alpha=(\alpha/2)C_k$, where $C_k\coloneqq\operatorname{MMD}_k(R_0,R_1)>0$. The transformation $T$ is constructed by uniformizing an appropriately centered latent reference distribution while preserving the class-conditional structure. The construction and proof are provided in Appendix~\ref{app:thm-proof}.
\end{remark}

Theorem~\ref{thm:reformulation} is a population-level result. In practice, the sample is finite and the learned encoder realizes the uniformizing map $T$ only approximately, so we introduce a calibration step that turns the reformulation into a practical test. Let $\mathcal{D}_{\rm cal}$ be a labeled sample from $P_0$ of size $n_{\rm cal}=|\mathcal{D}_{\rm cal}|$, disjoint from the split used to fit $\widehat\phi$, and let $\widehat R_0$ and $\widehat R_1$ be the empirical class-conditional latent distributions it induces. We form the reference $\widehat Q_{\alpha/2}\coloneqq(1-\alpha/2)\widehat R_0+(\alpha/2)\widehat R_1$, placing it at the midpoint of the safe-mixture range. For deployment contexts $X_{n+1},\ldots,X_{n+m}\sim P_1$, let $\widehat Q_{P_1}\coloneqq m^{-1}\sum_{j=n+1}^{n+m}\delta_{\widehat\phi(X_j)}$. The practical test statistic is $\widehat T_{\rm test}\coloneqq D(\widehat Q_{P_1},\widehat Q_{\alpha/2})$. For each $\lambda\in\mathcal G_\alpha\subset[0,\alpha]$ we draw $B$ deployment-sized batches from $(1-\lambda)\widehat R_0+\lambda\widehat R_1$, compare each with that same reference, and take $\widehat r_{\alpha,\beta}$ to be the largest resulting $(1-\beta)$-quantile, where $\beta$ is the significance level of the test, distinct from the violation tolerance $\alpha$; we reject when $\widehat T_{\rm test}>\widehat r_{\alpha,\beta}$. Full details are in Appendix~\ref{app:test-algo}.

\begin{theorem}[Calibrated Latent Test]
\label{thm:testing_performance}
Condition on the trained encoder $\widehat\phi$, and suppose the calibration conditions of Appendix~\ref{app:test-algo} hold. Let $h(\mathcal G_\alpha)\coloneqq\sup_{\lambda\in[0,\alpha]}\min_{\lambda'\in\mathcal G_\alpha}|\lambda-\lambda'|$ denote the calibration-grid resolution. For any $P_1$ satisfying Assumption~\ref{ass:conditional-invariance}, as $n_{\rm cal}\to\infty$, $B\to\infty$ and $h(\mathcal G_\alpha)\to0$:
\begin{itemize}[leftmargin=*,nosep]
\item If $P_1\in\calP_0$, for every fixed deployment size $m$,
\[
\limsup\,\mathbb P_{P_1}\!\left(\widehat T_{\rm test}>\widehat r_{\alpha,\beta}\right)\le\beta.
\qquad\text{(Type-I error control)}
\]
\item If $P_1\notin\calP_0$, \ie, $p_1\coloneqq\mathbb P_{P_1}(y(X)=1)>\alpha$, and $D(R_0,R_1)>0$, then
\[
\mathbb P_{P_1}\!\left(\widehat T_{\rm test}>\widehat r_{\alpha,\beta}\right)\to1
\quad\text{as } m\to\infty.
\qquad\text{(Power)}
\]
\end{itemize}
\end{theorem}

\begin{figure}[!t]
\centering
\begin{subfigure}{\linewidth}
\centering
\includegraphics[width=\linewidth]{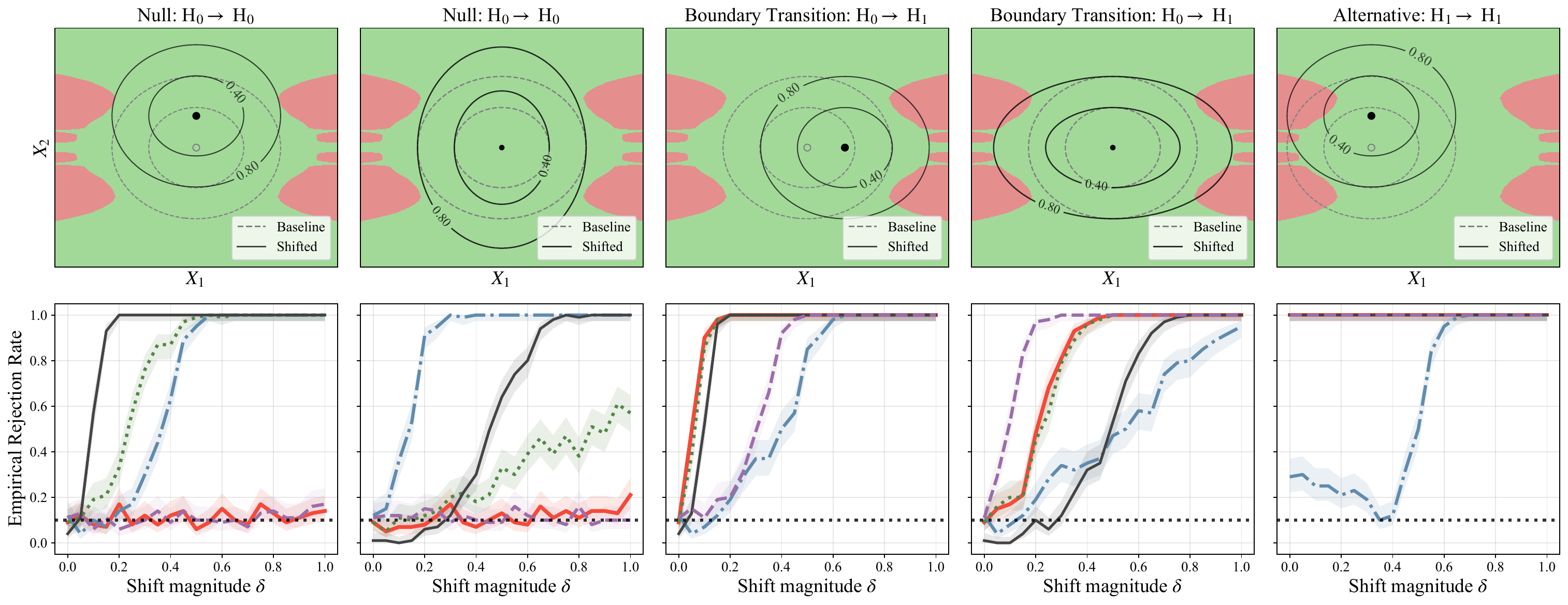}
\caption{Mass--spring--damper system.
}
\label{fig:spring_result}
\end{subfigure}
\begin{subfigure}{\linewidth}
\centering
\includegraphics[width=\linewidth]{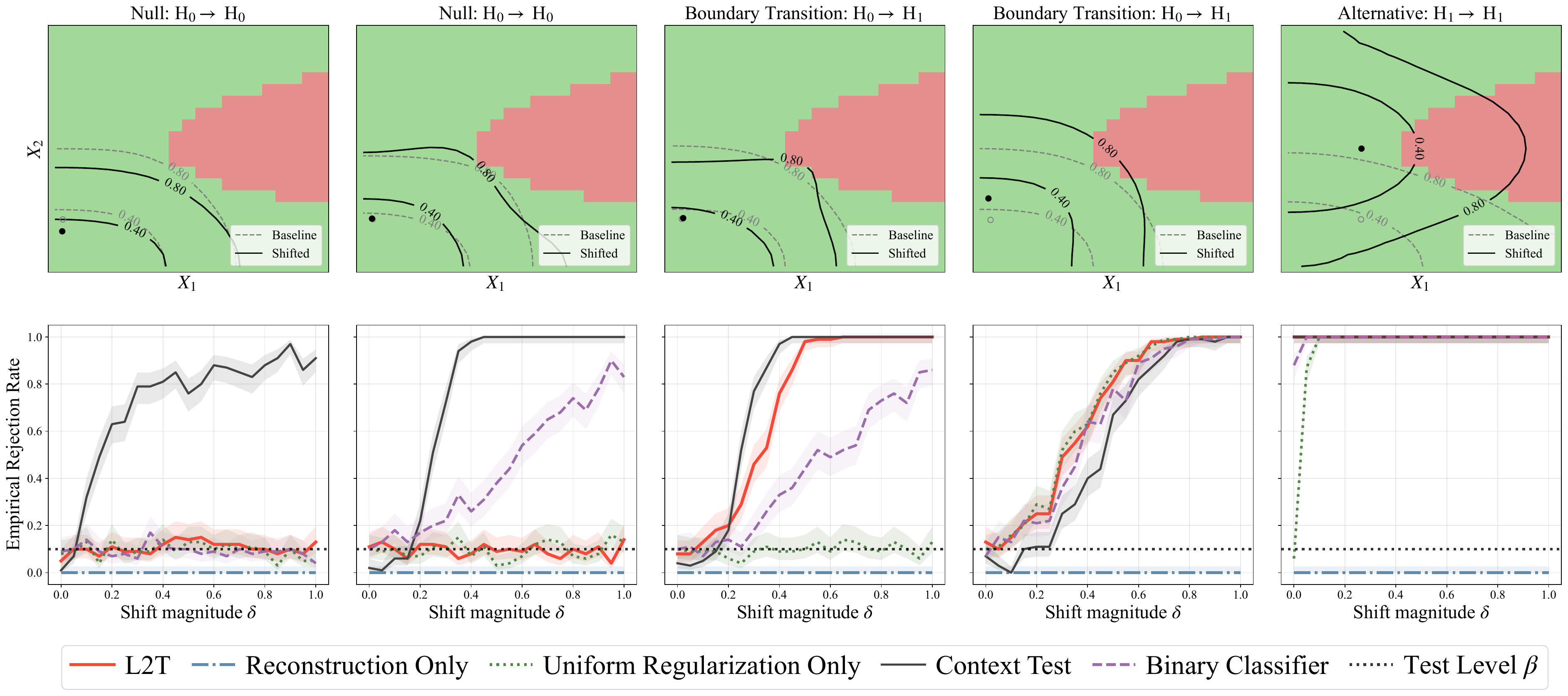}
\caption{Cartesian pendulum system.
}
\label{fig:pendulum_result}
\end{subfigure}
\caption{
Comparison between \texttt{L2T} and other baselines under various distribution shifts. \textcolor{mygreen}{Green} and \textcolor{myred}{red} regions denote stable and unstable contexts, respectively, while dashed gray and solid black contours represent the baseline and shifted distributions.
The top row of each panel shows baseline ($P_0$) and shifted ($P_1$) context distributions in the $2$D space across representative regimes: null ($H_0 \to H_0$), boundary transition ($H_0 \to H_1$), and alternative ($H_1 \to H_1$).
The bottom row shows empirical rejection rates over $100$ runs against the shift magnitude, with $90\%$ Wilson intervals shaded.
The dotted line marks the test level $\beta=0.1$.
}
  \label{fig:synthetic_results_combined}

\end{figure}

\section{Experiments}
\label{sec:experiment}

We evaluate \texttt{L2T} on two simulated systems, a mass--spring--damper and a Cartesian pendulum, and on a power-system benchmark generated from physics-based simulations. Stability is defined by trajectory-level functionals (maximum displacement or swing angle), and we apply controlled mean, covariance, marginal, and dependence shifts to the context distribution, covering null, boundary-transition, and alternative regimes.

For the power-system benchmark, we use the IEEE 39-bus transient stability assessment (TSA) dataset \citep{Sokolovic2024TSA}: $12{,}852$ simulated trajectories of generator-level signals ($50$ features) under varying operating conditions and fault scenarios, each labeled stable or unstable. We train on a $95{:}5$ stable-to-unstable mixture, reflecting the practical setting in which training data come from the safe regime and violations are rare. In Definition~\ref{def:stab}, we set the chance-constraint tolerance to $\alpha=0.1$ for the synthetic systems and $\alpha=0.05$ for TSA; full details are in Appendix~\ref{app:exp-conf-results}.

\paragraph{Baselines and Configurations.}
We compare \texttt{L2T} with two alternative detection strategies and two ablations of its training objective. ($i$) \emph{Context Test}: a two-sample test on the raw contexts $X$ that ignores system dynamics, using MMD \citep{gretton2012kernel} in low dimensions and a projection-based test in high dimensions. ($ii$) \emph{Binary classifier}: a model that predicts trajectory stability from $X$, followed by a two-sample mean test on the predicted stability probabilities between baseline and deployment samples. ($iii$) \emph{Uniform-only} and ($iv$) \emph{Reconstruction-only}: latent representations trained with only the uniformity regularization or only trajectory reconstruction, which isolate the contributions of latent geometry and of dynamics-aware modeling. Appendix~\ref{app:synthetic} further compares \texttt{L2T} with oracle simulation of the deployment trajectories and with a trajectory surrogate followed by a label count test (Figure~\ref{fig:oracle_simulation_size}). Unless otherwise specified, synthetic experiments use MMD-based uniformity regularization, all tests use significance level $\beta=0.1$, and we report empirical rejection rates across shift magnitudes; implementation details are in Appendix~\ref{app:exp-conf-results}.

\paragraph{Results on Simulation Studies.}
Figure~\ref{fig:synthetic_results_combined} shows the rejection curves for both systems, and Table~\ref{tab:result_summary} summarizes them. \texttt{L2T} keeps its rejection rate near the nominal level $\beta=0.1$ under the null shifts while attaining low Type-II error, and hence high power, in the boundary-transition regime. In contrast, each competing method fails in at least one regime. The context test over-rejects benign shifts on both systems, since it reacts to any change in the context distribution. Reconstruction-only over-rejects on the spring and misses instability on the pendulum, indicating that accurate trajectory reconstruction alone does not yield a representation suited to stability testing. Uniform-only also over-rejects on the spring and loses power on the pendulum, reflecting the importance of dynamics-aware structure. The binary classifier is competitive on the spring but over-rejects and loses power on the pendulum. Together, these comparisons indicate that \emph{effective stability testing requires both dynamics-aware representation learning and test-aligned latent geometry}. Appendix~\ref{app:synthetic} adds loss-weight sensitivity analyses and latent-density visualizations (Figures~\ref{fig:sensitivity_simplex} and~\ref{fig:latent_space_visual}).

\begin{table}[!t]

\centering
\resizebox{\linewidth}{!}{
\begin{threeparttable}
\caption{Comparison of empirical Type-I and Type-II errors at test level $\beta=0.1$ (mean $\pm$ standard error).}
\label{tab:result_summary}
\renewcommand{\arraystretch}{1.15}
\setlength{\tabcolsep}{5pt}
\begin{tabular}{llccccc}
\toprule
\textbf{Dataset} & \textbf{Metric}
& \textsc{Context Test}
& \textsc{Binary-clf.}
& \textsc{Recon-only}
& \textsc{Uniform-only}
& {\texttt{L2T}} \\
\midrule
\multirow{2}{*}{Mass--spring--damper}
& Type-I $\downarrow$
& $1.00 \pm 0.00$
& $0.17 \pm 0.04$
& $1.00 \pm 0.00$
& $1.00 \pm 0.00$
& $0.14 \pm 0.03$ \\
& Type-II $\downarrow$
& $0.47 \pm 0.05$
& $0.00 \pm 0.00$
& $0.54 \pm 0.05$
& $0.00 \pm 0.00$
& $0.00 \pm 0.00$ \\
\midrule
\multirow{2}{*}{Cartesian pendulum}
& Type-I $\downarrow$
& $1.00 \pm 0.00$
& $0.83 \pm 0.04$
& $0.00 \pm 0.00$
& $0.13 \pm 0.03$
& $0.13 \pm 0.03$ \\
& Type-II $\downarrow$
& $0.00 \pm 0.00$
& $0.56 \pm 0.05$
& $1.00 \pm 0.00$
& $0.90 \pm 0.03$
& $0.02 \pm 0.01$ \\
\midrule
\multirow{2}{*}{IEEE 39-bus TSA}
& Type-I $\downarrow$
& $0.01 \pm 0.01$
& $1.00 \pm 0.00$
& $0.11 \pm 0.02$
& $0.12 \pm 0.02$
& $0.10 \pm 0.02$ \\
& Type-II $\downarrow$
& $0.18 \pm 0.03$
& $0.00 \pm 0.00$
& $0.00 \pm 0.00$
& $0.71 \pm 0.03$
& $0.01 \pm 0.01$ \\
\bottomrule
\end{tabular}
\begin{tablenotes}[flushleft]
\footnotesize
\item Type-I error is reported under the null, and Type-II error is reported at a representative shift magnitude in the boundary-transition regime. For Type-I error, values closer to level $0.1$ indicate better control; for Type-II error, smaller values indicate higher detection power.
\end{tablenotes}
\end{threeparttable}
}

\end{table}

\paragraph{Results on TSA Benchmark.}
Because the TSA contexts cannot be shifted directly, we vary the unstable proportion of the test set from $0\%$ to $30\%$, a label shift under which Assumption~\ref{ass:conditional-invariance} holds by construction. Table~\ref{tab:result_summary} uses the $100{:}0$ stable-to-unstable mixture as the null and the $80{:}20$ mixture as the alternative, with \texttt{L2T} regularized by a normalizing flow on an $8$-dimensional latent. \texttt{L2T} attains Type-I error at the nominal level and near-zero Type-II error, whereas the binary classifier severely over-rejects under the null and the context and uniform-only tests lose substantial power; reconstruction-only performs on par with \texttt{L2T} on this benchmark. Appendix~\ref{app:real-data} reports continuous shift curves (Figure~\ref{fig:tsa_delta_shift}) and shows that MMD regularization with a one-dimensional latent is markedly less powerful (Table~\ref{tab:tsa_results_updated}), suggesting that flow regularization with a richer latent suits the high-dimensional TSA dynamics better.

\section{Conclusion and Future Work}

We proposed \texttt{L2T}, which learns a latent representation from baseline context--trajectory data and tests a trajectory-level chance constraint from deployment contexts alone; on two simulated systems and a power-system benchmark it balances Type-I error and power. It targets settings where deployment simulation is too costly; when simulation is affordable, a count test on simulated trajectories remains the reference (Appendix~\ref{app:synthetic}). Its main limitation is Assumption~\ref{ass:conditional-invariance}, which holds by construction under shifts of the class proportions but only approximately under covariate shift: drift smaller than the safety margin is harmless whenever $p_1\ge\alpha/2$ (Proposition~\ref{prop:conditional-drift}), whereas near-boundary decisions under strong covariate shift remain sensitive to it. Calibration also needs labeled baseline data from both classes and is valid only asymptotically in their number, and power depends on how well the latent separates the classes. Future work includes bounding the drift from a small labeled deployment sample, learning representations whose class-conditional laws are invariant by design, and sequential monitoring, which updates the test as deployment contexts arrive.

\section*{Acknowledgments}

We thank Prof.~Bai Cui (Iowa State University) for insightful discussions that shaped the problem formulation and helped us identify relevant real-world applications.

\bibliographystyle{plainnat}
\bibliography{ref}

\appendix
\newpage

\section{Motivating Examples}
\label{app:motivating-exps}

We highlight two application domains in which the formulation in \eqref{eq:hypothesis} arises naturally. In both, the dynamics follow physical laws with unknown parameters, the operating context shifts over time, and trajectories under a new regime are costly or risky to observe, so stability must be judged from baseline trajectories and deployment contexts alone.

\noindent\emph{Power systems with stochastic loads or generations}.
Electric power grids must maintain frequency, voltage, and thermal limits while accommodating stochastic demand and renewable generation. Their dynamics are described by the generator swing equations coupled with algebraic power-flow constraints \citep{anderson2008power, kundur1994power, sauer2017power}. The state $s(t)$ includes generator rotor angles and frequencies, while algebraic variables $a(t)$ enforce network power balance and voltage constraints. The exogenous context $X$ captures load profiles, renewable outputs, or network configurations.
A baseline domain corresponds to normal operating conditions under which stability margins are well established. When the context distribution shifts---due to large data-center loads, renewable intermittency, or extreme weather---new operating conditions arise. Full dynamic trajectories under these regimes are typically unavailable, as large-scale disturbance experiments or exhaustive simulations are costly and potentially risky. Instead, operators often observe only contextual information such as load forecasts or topology changes. The proposed framework evaluates whether the probability of violating stability limits over a time horizon remains below a prescribed level under the shifted context distribution, enabling statistically grounded security assessment under emerging grid conditions.

\noindent\emph{Material degradation and structural integrity}.
Material degradation processes such as corrosion and fatigue crack growth are governed by physical laws that couple time-evolving damage dynamics with thermodynamic and mechanical constraints. These processes are commonly modeled by differential equations for damage evolution together with algebraic constitutive or equilibrium relations \citep{kacher2022insitu, landolt2007corrosion,macdonald1992point,suresh1998fatigue}. The state $s(t)$ may represent corrosion depth or crack length, while the exogenous context $X$ encodes environmental and usage conditions such as humidity, temperature, salinity, or load cycles.
A baseline regime corresponds to operating conditions where degradation behavior and safety margins are well understood. When the environment changes, the distribution of $X$ shifts, yet full degradation trajectories under the new regime are typically unavailable without long-term monitoring or destructive testing. Our framework asks whether the probability that degradation exceeds a critical threshold over a given horizon remains below a target level under the new context distribution, enabling uncertainty-aware maintenance and inspection decisions.

\section{Implementation Details: Representation Learning for Stability Testing}
\label{append:learning-details}

This section describes practical implementations of the three components in the loss function \eqref{eq:loss_total}. Throughout, we train the parameters of the encoder $\phi$, decoder $\psi$, and update function $\varphi$ using mini-batch stochastic gradient methods. Each iteration samples a mini-batch $\mathcal B=\{(x_i,\mathcal S_i)\}_{i\in\mathcal I}$ from the baseline data, computes latent representations $u_i=\phi(x_i)$, rolls out the latent dynamics in \eqref{eq:h-embedding} to obtain reconstructions, evaluates the losses on the mini-batch, and updates parameters by backpropagation through the unrolled computation graph. We use the MMD regularizer in the simulation experiments, where the latent is one-dimensional, and the normalizing-flow regularizer on the TSA benchmark, where the latent has eight dimensions.

\emph{Reconstruction loss in \eqref{eq:loss-reconstruction}}: The sample size $K_i$ may vary across trajectories (irregular sampling). We either ($i$) pad shorter sequences and apply a mask, or ($ii$) subsample a fixed-length window per trajectory per iteration. Gradients are computed by differentiating through $\hat s_{i,k}=\psi(h_{i,k})$ and the recurrent updates $h_{i,k+1}=\varphi(s_{i,k},h_{i,k},u_i)$.

\emph{Uniformity regularization in \eqref{eq:loss-uniform}}:
We encourage the latent representations $u=\phi(x)\in[0,1]^d$ to follow $\mathrm{Unif}(0,1)^d$, so that downstream stability testing reduces to detecting deviations from uniformity in the new domain. We present two practical implementations.

\paragraph{Maximum Mean Discrepancy (MMD) to Uniform.}

Let $\{u_i\}_{i=1}^B$ be latent representations from a mini-batch with $B$ being the sample size of the mini-batch, and let $\{v_j\}_{j=1}^B$ be i.i.d. samples from $v_j\sim \mathrm{Unif}(0,1)^d$. The squared maximum mean discrepancy (MMD) \citep{gretton2012kernel} between $p_U$ and the uniform distribution is
\[
\mathrm{MMD}^2(p_U,\mathrm{Unif})
=
\Big\|\mathbb E_{u\sim p_U}[\kappa(u,\cdot)]-\mathbb E_{v\sim \mathrm{Unif}}[\kappa(v,\cdot)]\Big\|_{\mathscr H_\kappa}^2,
\]
where $\kappa$ is a positive definite kernel and $\mathscr H_\kappa$ is the associated reproducing kernel Hilbert space. Using unbiased (or biased) V-statistic estimators, we implement
\begin{align*}
\mathcal L_{\mathrm{reg}}^{\mathrm{MMD}}
&\coloneqq
\widehat{\mathrm{MMD}}^2(\{u_i\}_{i=1}^B,\{v_j\}_{j=1}^B)\\
&=
\frac{1}{B^2}\sum_{i=1}^B\sum_{i'=1}^B \kappa(u_i,u_{i'})
+
\frac{1}{B^2}\sum_{j=1}^B\sum_{j'=1}^B \kappa(v_j,v_{j'})
-
\frac{2}{B^2}\sum_{i=1}^B\sum_{j=1}^B \kappa(u_i,v_j).
\end{align*}
A common choice is the Gaussian kernel $\kappa(u,u')=\exp(-\|u-u'\|^2/(2\sigma^2))$ with bandwidth $\sigma$ chosen by a median heuristic or a small grid. This loss is fully differentiable in $\{u_i\}$, hence differentiable in $\phi$ via $u_i=\phi(x_i)$.

At each iteration, we proceed as follows: ($i$) compute $u_i=\phi(x_i)$ for $i=1,\dots,B$; ($ii$) sample $v_j\sim\mathrm{Unif}(0,1)^d$ for $j=1,\dots,B$; ($iii$) evaluate the kernel matrices $K_{uu}=[\kappa(u_i,u_{i'})]$, $K_{vv}=[\kappa(v_j,v_{j'})]$, $K_{uv}=[\kappa(u_i,v_j)]$; ($iv$) compute $\mathcal L_{\mathrm{reg}}^{\mathrm{MMD}}$ via the above V-statistic and backpropagate.
This option is typically stable and requires no density estimation when $d$ is small.

\paragraph{Direct Entropy Maximization via a Normalizing Flow.}

An alternative is to model the latent distribution explicitly using a bijective normalizing flow \citep{lipman2022flow, rezende2015variational}. Let $v\in\mathbb R^d$ be a base random variable with known density $p_V$ (\eg, uniform). Let $\mathcal{F}_\beta:\mathbb R^d\to(0,1)^d$ be an invertible map parameterized by $\beta$, and define the latent code as
\[
u = \mathcal{F}_\beta(v), \qquad v\sim p_V.
\]
The induced density of $u$ is given by the change-of-variables formula:
\begin{equation}
\label{eq:cov_flow}
p_U(u) = p_V\big(\mathcal{F}_\beta^{-1}(u)\big)\cdot
\left|\det\nabla_u \mathcal{F}_\beta^{-1}(u)\right|.
\end{equation}
To couple this with the context encoder, we
parameterize $\phi$ itself as a flow mapping from a simple base distribution conditioned on $x$. A simple and effective approach is to fit a flow to the empirical latent representations and use the flow log-density to approximate differential entropy.

The differential entropy satisfies $\mathcal H(p_U)=-\mathbb E_{u\sim p_U}[\log p_U(u)]$. Using \eqref{eq:cov_flow}, we obtain
\begin{align}
\label{eq:entropy_flow}
\mathcal H(p_U)
&= -\mathbb E_{u\sim p_U}\big[\log p_U(u)\big] \nonumber\\
&= -\mathbb E_{u\sim p_U}\Big[\log p_V\big(\mathcal{F}_\beta^{-1}(u)\big) +\log\big|\det\nabla_u \mathcal{F}_\beta^{-1}(u)\big|\Big].
\end{align}
Thus, maximizing entropy is equivalent to minimizing the negative log-likelihood under the flow model. In practice, given a mini-batch $\{u_i\}_{i=1}^B$, we implement
\[
\mathcal L_{\mathrm{reg}}^{\mathrm{flow}} \coloneqq \frac{1}{B}\sum_{i=1}^B \Big[ -\log p_V(v_i) -\log\big|\det\nabla_{u} \mathcal{F}_\beta^{-1}(u_i)\big| \Big], \qquad v_i=\mathcal{F}_\beta^{-1}(u_i),
\]
and differentiate through both $\phi$ (via $u_i=\phi(x_i)$) and the flow parameters $\beta$.

To explicitly enforce $\mathrm{Unif}(0,1)^d$, we can choose the base density $p_V$ to be uniform and select $\mathcal{F}_\beta$ as an invertible transformation on $(0,1)^d$. In this case, minimizing $\mathcal L_{\mathrm{reg}}^{\mathrm{flow}}$ encourages $\{u_i\}$ to match the target law by making them high-likelihood under a flow whose output is constrained to $[0,1]^d$.

At each iteration, we proceed as follows: ($i$) compute $u_i=\phi(x_i)$ for $i=1,\dots,B$; ($ii$) evaluate $v_i=\mathcal{F}_\beta^{-1}(u_i)$ and $\log|\det\nabla_u \mathcal{F}_\beta^{-1}(u_i)|$ via the flow; ($iii$) compute $\mathcal L_{\mathrm{reg}}^{\mathrm{flow}}$ and backpropagate through $(\phi,\beta)$.
Compared with the MMD-based approach, the flow-based formulation provides an explicit parametric density model for $p_U$ and a direct estimate of differential entropy. It scales more naturally to higher-dimensional latent spaces, albeit at the cost of increased computational overhead and additional modeling complexity.

\section{Details of Composite-Null Calibration}
\label{app:test-algo}

The null hypothesis $H_0:P_1\in\calP_0$ is composite: it contains every deployment law whose violation probability is at most $\alpha$, not only the baseline law. A two-sample test against the baseline therefore rejects safe deployments that merely differ from $P_0$, for instance those with fewer violations. Under Assumption~\ref{ass:conditional-invariance}, every null law induces a latent mixture $Q_\lambda=(1-\lambda)R_0+\lambda R_1$ with $\lambda\in[0,\alpha]$, so it suffices to simulate these mixtures from labeled baseline data and to take the largest null quantile over $\lambda$. Centering the reference at $\lambda=\alpha/2$ places every alternative farther from the reference than any null law (Theorem~\ref{thm:reformulation}).

Theorem~\ref{thm:testing_performance} calibrates the statistic over the entire
composite null $P\in\calP_0$. To approximate this calibration for a learned
encoder, reserve an independent labeled baseline split and form the empirical
conditional latent pools
\[
\widehat{\mathscr R}_0
=\{\widehat\phi(X_i):y(X_i)=0\},
\qquad
\widehat{\mathscr R}_1
=\{\widehat\phi(X_i):y(X_i)=1\}.
\]
Let $\widehat R_a$ denote the empirical distribution supported on
$\widehat{\mathscr R}_a$. Construct the fixed empirical reference $\widehat Q^*=(1-\alpha/2)\widehat R_0+(\alpha/2)\widehat R_1$. For each grid value $\lambda\in\mathcal G_\alpha\subset[0,\alpha]$, resample batches of size $m$ from $\widehat Q_\lambda=(1-\lambda)\widehat R_0+\lambda\widehat R_1$, compute their discrepancy to the same $\widehat Q^*$ used at deployment, and retain the largest empirical null quantile. The reference and calibration batches use the same labeled split; independence is required between this split and the deployment sample, not between the two calibration calculations.

The calibration result conditions on the trained encoder and assumes: ($i$)
the calibration and deployment samples are independent, and both class-conditional calibration sample sizes diverge; ($ii$) $\widehat R_a$ consistently estimates $R_a$ and $\widehat Q^*$ consistently estimates the centered mixture $Q_{\alpha/2}=(1-\alpha/2)R_0+(\alpha/2)R_1$ in the topology induced by $D$; ($iii$) $B\to\infty$ and $h(\mathcal G_\alpha)\to0$; and ($iv$) for each fixed deployment size $m$, the quantile of the two-sample statistic is continuous in $\lambda$ and jointly consistent when the class-conditional laws and the reference are replaced by their empirical estimates, uniformly over the grid. For consistency, we additionally assume that $D(\widehat Q_{P_1},\widehat Q^*)\to D(Q_{P_1},Q_{\alpha/2})$ and that the null quantiles converge uniformly over $\lambda\in[0,\alpha]$ as $m\to\infty$. These conditions hold, for example, for MMD with a bounded continuous kernel under the usual quantile-continuity condition; a characteristic kernel additionally ensures distributional separation.
\begin{algorithm}[!t]
\caption{Learning-to-Test (\texttt{L2T}) with composite-null calibration}
\label{alg:l2t_pipeline_composite}
\begin{algorithmic}[1]
\Require Baseline train/calibration data; new contexts $\{x'_j\}_{j=1}^{m}$; $D,\alpha,\beta$; grid $\mathcal G_\alpha\subset[0,\alpha]$; repetitions $B$.
\State Train $\widehat\phi$ on the training split and freeze it.
\State Form $\widehat{\mathscr R}_0$ and $\widehat{\mathscr R}_1$ from the independent labeled calibration split.
\State $\widehat Q^*\leftarrow(1-\alpha/2)\widehat R_0+(\alpha/2)\widehat R_1$; keep this reference fixed.
\For{$\lambda\in\mathcal G_\alpha$}
    \For{$b=1,\ldots,B$}
        \State Draw a batch of size $m$ from $(1-\lambda)\widehat R_0+\lambda\widehat R_1$ and form its empirical law $\widehat Q_\lambda^{(b)}$.
        \State $\widehat{\mathcal T}_\lambda^{(b)}\leftarrow D(\widehat Q_\lambda^{(b)},\widehat Q^*)$.
    \EndFor
    \State $\widehat r_\beta(\lambda)\leftarrow\operatorname{Quantile}_{1-\beta}\{\widehat{\mathcal T}_\lambda^{(b)}\}_{b=1}^B$.
\EndFor
\State $\widehat r_{\alpha,\beta}\leftarrow\max_{\lambda\in\mathcal G_\alpha}\widehat r_\beta(\lambda)$.
\State $\widehat Q_{P_1}\leftarrow m^{-1}\sum_{j=1}^{m}\delta_{\widehat\phi(x'_j)}$ and $\widehat{\mathcal T}_{\rm test}\leftarrow D(\widehat Q_{P_1},\widehat Q^*)$.
\State Reject $H_0:P_1\in\calP_0$ if $\widehat{\mathcal T}_{\rm test}>\widehat r_{\alpha,\beta}$.
\State \Return $\mathbbm{1}\{\widehat{\mathcal T}_{\rm test}>\widehat r_{\alpha,\beta}\}$.
\end{algorithmic}
\end{algorithm}

\begin{remark}
\label{remark:two-sample-test-critical-value}
The oracle supremum in \eqref{eq:composite-critical-value}, which uses the population class-conditional laws and centered reference, has exact finite-sample validity for each fixed deployment size. The implemented statistic estimates both the reference and the calibration mixtures from the same labeled split; its stated validity is asymptotic. Finite calibration, Monte Carlo, and grid errors can be reduced with more calibration observations, more repetitions, a finer grid, or an additional conservative correction, and, for small calibration samples, by drawing the reference and the calibration batches from independent halves of the labeled split. The interpretation also requires the conditional latent laws to remain representative under the deployment shifts under consideration. Here $\alpha$ denotes the allowed violation probability, whereas $\beta$ denotes the test significance level.
\end{remark}

\begin{remark}[Reference choice and power]
\label{rem:reference-power}
For $Q_\lambda=(1-\lambda)R_0+\lambda R_1$, the centered reference $Q_{\alpha/2}$ separates every $P_1\notin\calP_0$ from the composite null under a separating IPM. Using the unweighted baseline law $Q_{P_0}$ as the reference would still permit composite-null calibration, but if $\mathbb P_{P_0}(y(X)=1)>\alpha/2$, some alternatives need not be separated from the null. This is why the empirical two-sample test constructs $\widehat Q^*$ with mixing weight $\alpha/2$ rather than using the raw baseline mixture.
\end{remark}

\section{Theoretical Results and Proofs}
\label{app:theory}

This section collects the theory. Appendix~\ref{app:nonident} shows that some assumption linking the baseline and deployment domains is necessary, and Appendix~\ref{app:thm-proof} proves Theorem~\ref{thm:reformulation}. Appendix~\ref{app:approx} quantifies what happens when Assumption~\ref{ass:conditional-invariance} holds only approximately and illustrates the result, and Appendix~\ref{app:testing-performance} proves Theorem~\ref{thm:testing_performance}.

\subsection{Why a Link Between the Domains Is Necessary}
\label{app:nonident}

The observed data consist of baseline context--trajectory pairs, whose law is determined by $P_0$ and by the trajectories $s_\theta(\cdot;x)$ for contexts $x$ in the support $\mathcal S_0$ of $P_0$, and of deployment contexts, whose law is $P_1$. The next proposition shows that these data cannot settle \eqref{eq:hypothesis} once $P_1$ places enough mass where the baseline domain provides no trajectories, so some assumption linking the two domains is unavoidable.

\begin{proposition}[Non-identifiability without a link between the domains]
\label{prop:nonident}
Let $\mathcal S_0$ be the support of $P_0$ and suppose that $P_1(\scrX\setminus\mathcal S_0)>\alpha$. Then there exist two systems whose trajectories coincide for every context in $\mathcal S_0$, such that $P_1\in\calP_0$ for the first and $P_1\notin\calP_0$ for the second. The observed data have the same law under both systems, so any test of \eqref{eq:hypothesis} based on these data rejects with the same probability under both; a test with level $\beta$ under the first has power at most $\beta$ under the second.
\end{proposition}

\begin{proof}
Let the first system never violate the tolerance, so that $\mathbb P_{P_1}(y(X)=1)=0\le\alpha$. Let the second system have the same trajectories on $\mathcal S_0$ and violate the tolerance for every context outside $\mathcal S_0$, so that $\mathbb P_{P_1}(y(X)=1)\ge P_1(\scrX\setminus\mathcal S_0)>\alpha$. Such a pair exists within model \eqref{eq:dae_main}: because $f_\theta$ may depend on the context, adding to the first system's $f_\theta$ a forcing term that is active only for $x\notin\mathcal S_0$ and drives the trajectory beyond the tolerance yields the second system; the added term is constant in $(s,a)$, so smoothness and well-posedness are preserved. The baseline data depend on the trajectories only through contexts in $\mathcal S_0$, and the deployment data contain no trajectories, so the observed data have the same law under both systems. A test is a function of these data, so its rejection probability is the same under both systems.
\end{proof}

Figure~\ref{fig:assumption}(a, b) illustrates the construction. Support overlap together with a model for the density ratio $\d P_1/\d P_0$ is one way to link the domains; Assumption~\ref{ass:conditional-invariance} is another, imposed on the learned representation, and Proposition~\ref{prop:conditional-drift} quantifies how its approximate violation affects the test.

\subsection{Proof of Theorem~\ref{thm:reformulation}}
\label{app:thm-proof}

\paragraph{Full statement.} Theorem~\ref{thm:reformulation} holds under the following conditions, some of which the main text leaves implicit: ($i$) $0<\alpha<1$; ($ii$) the deployment law $P_1$ satisfies Assumption~\ref{ass:conditional-invariance}; ($iii$) $R_0\neq R_1$; ($iv$) the centered mixture $Q^\dagger\coloneqq(1-\alpha/2)R_0+(\alpha/2)R_1$ is non-atomic, which holds, for example, when $R_0$ and $R_1$ have densities on $[0,1]^d$; and ($v$) $D(\mu,\nu)=\sup_{f\in\mathcal F}|\int f\,\d(\mu-\nu)|$ is an integral probability metric that separates Borel probability measures on $[0,1]^d$. Condition ($iv$) is the regularity condition of the main text and is used only to construct $T$; conditions ($iii$) and ($v$) make the threshold $r_\alpha$ positive.

\paragraph{Intuition.} Under Assumption~\ref{ass:conditional-invariance}, every deployment law induces a latent law on the segment between $R_0$ and $R_1$, at the position given by its violation probability. The null is the part of the segment with violation probability at most $\alpha$. Centering the reference at the midpoint $\alpha/2$ of this part turns the null into a ball of radius $(\alpha/2)D(R_0,R_1)$ around the reference, and a transformation of the latent makes the reference uniform.

\begin{proof}
We consider the nondegenerate case $0<\alpha<1$.
For a deployment law $P$ that satisfies Assumption~\ref{ass:conditional-invariance}, write
$\pi_P\coloneqq\mathbb P_{X\sim P}(y(X)=1)$. The assumption
implies
\begin{equation}
\label{eq:aligned-mixtures}
Q_P=(1-\pi_P)R_0+\pi_P R_1.
\end{equation}
If a class has probability zero, its conditional law is not needed because its mixture weight is zero.
Thus, all admissible deployment laws induce distributions on the line segment
joining $R_0$ and $R_1$, with the violation probability $\pi_P$ as the mixing
weight.

Define the centered mixture
\[
Q^\dagger
\coloneqq
\left(1-\frac{\alpha}{2}\right)R_0+\frac{\alpha}{2}R_1.
\]
By condition ($iv$), $Q^\dagger$ is non-atomic.
Because $[0,1]^d$ is a standard Borel space, the
isomorphism theorem for non-atomic Borel probability measures yields a
measurable map $T:[0,1]^d\to[0,1]^d$ that is injective outside a
$Q^\dagger$-null set and satisfies
\[
T_\#Q^\dagger=Q^*\coloneqq\mathrm{Unif}([0,1]^d).
\]
For $d=1$, the probability integral transform gives such a map, up to changes
on a $Q^\dagger$-null set.

Let
\[
\widetilde R_a\coloneqq T_\#R_a,
\qquad
\widetilde Q_P\coloneqq T_\#Q_P,
\qquad a\in\{0,1\}.
\]
Applying the same transformation to both classes preserves
Assumption~\ref{ass:conditional-invariance}, and \eqref{eq:aligned-mixtures}
becomes
\[
\widetilde Q_P=(1-\pi_P)\widetilde R_0+\pi_P\widetilde R_1,
\qquad
Q^*=\left(1-\frac{\alpha}{2}\right)\widetilde R_0
+\frac{\alpha}{2}\widetilde R_1.
\]
Moreover, $R_a\ll Q^\dagger$ for $a\in\{0,1\}$, because both mixture weights of $Q^\dagger$ are positive. Let $E$ be a Borel set of full $Q^\dagger$-measure on which $T$ is injective; it also has full $R_a$-measure. By the Lusin--Souslin theorem, $T$ maps Borel subsets of $E$ to Borel sets, and injectivity gives $T^{-1}(T(B))\cap E=B$, so $R_a(B)=\widetilde R_a(T(B))$ for every Borel $B\subseteq E$. Hence equality
$\widetilde R_0=\widetilde R_1$ would imply $R_0=R_1$, and therefore
$\widetilde R_0\neq\widetilde R_1$.

Subtracting the two mixture representations gives
\[
\widetilde Q_P-Q^*
=\left(\pi_P-\frac{\alpha}{2}\right)
(\widetilde R_1-\widetilde R_0).
\]
For an integral probability metric
$D(\mu,\nu)=\sup_{f\in\mathcal F}|\int f\,\d(\mu-\nu)|$, absolute
homogeneity in the signed measure $\mu-\nu$ therefore yields
\begin{equation}
\label{eq:ipm-mixture-identity}
D(\widetilde Q_P,Q^*)
=\left|\pi_P-\frac{\alpha}{2}\right|
D(\widetilde R_0,\widetilde R_1).
\end{equation}
Because $D$ is separating and $\widetilde R_0\neq\widetilde R_1$, the final
factor is strictly positive. Since $\pi_P\in[0,1]$,
\[
\pi_P\le\alpha
\quad\Longleftrightarrow\quad
\left|\pi_P-\frac{\alpha}{2}\right|\le\frac{\alpha}{2}
\quad\Longleftrightarrow\quad
D(\widetilde Q_P,Q^*)\le r_\alpha,
\]
where
\[
r_\alpha
\coloneqq
\frac{\alpha}{2}D(\widetilde R_0,\widetilde R_1)>0.
\]
Taking $P=P_1$ proves \eqref{eq:reform-test}. After replacing $\phi$ by
$\widetilde\phi=T\circ\phi$, dropping the tildes recovers the notation used in
the theorem. In particular,
$r_\alpha=(\alpha/2)\operatorname{TV}(R_0,R_1)$ for total variation and
$r_\alpha=(\alpha/2)\operatorname{MMD}_k(R_0,R_1)$ for MMD with a
characteristic kernel. If squared MMD is used as the statistic, the
corresponding threshold is $r_\alpha^2$.
\end{proof}

\subsection{Robustness to Approximation}
\label{app:approx}

Assumption~\ref{ass:conditional-invariance} rarely holds exactly for a learned encoder. We consider two kinds of approximation: imperfect recovery of the label by the latent, and drift of the class-conditional latent laws. In both cases the guarantees degrade gracefully: the error enters additively and is at most the size of the approximation.

\subsubsection{Approximate Latent Sufficiency}

The main reformulation (Theorem~\ref{thm:reformulation}) rests on class-conditional invariance and does not require the latent to recover individual labels. We give here a complementary route through approximate label recovery, which applies without class-conditional invariance, and quantify how it degrades when the learned representation only approximately recovers $y(X)$. This result does not assume that approximate sufficiency alone produces uniformity or conditional reference alignment.

\begin{assumption}[Approximate latent sufficiency]
\label{ass:approx}
There is a measurable decision rule $\rho:[0,1]^d\to\{0,1\}$ such that, for $e\in\{0,1\}$, with $X\sim P_e$ and $U=\phi(X)$,
\[
\mathbb P_{P_e}\bigl(y(X)\ne\rho(U)\bigr)\le\varepsilon_e.
\]
\end{assumption}

\begin{theorem}[Robust latent reformulation]
\label{thm:approx_reform}
Under Assumption~\ref{ass:approx}, let $A_1=\{u:\rho(u)=1\}$,
$p_e=\mathbb E_{P_e}[y(X)]$, and $Q_e=P_e\circ\phi^{-1}$. Then
\begin{equation}
\label{eq:approx-mass-bound}
\left|p_e-Q_e(A_1)\right|\le\varepsilon_e,
\qquad e\in\{0,1\}.
\end{equation}
Consequently,
\begin{equation}
\label{eq:approx-two-sample-bound}
\mathrm{TV}(Q_1,Q_0)
\ge \left|p_1-p_0\right|-\varepsilon_0-\varepsilon_1.
\end{equation}
If the ideal reference-centering condition $Q^*(A_1)=\alpha/2$ also holds, then
\begin{equation}
\label{eq:approx-reference-bound}
\mathrm{TV}(Q_1,Q^*)
\ge \left|p_1-\frac{\alpha}{2}\right|-\varepsilon_1.
\end{equation}
In particular, $\mathrm{TV}(Q_1,Q^*)\le\alpha/2-\varepsilon_1$ certifies $p_1\le\alpha$. Moreover, whenever
$|p_1-p_0|>\varepsilon_0+\varepsilon_1$, we have $Q_1\ne Q_0$; a characteristic-kernel MMD can therefore detect this fixed alternative consistently.
\end{theorem}

\begin{proof}
Set $\widehat y=\rho(U)$. For binary variables, the difference between the probabilities of either event is bounded by their disagreement probability, so
\[
\left|p_e-Q_e(A_1)\right|
=\left|\mathbb P_{P_e}(y(X)=1)-\mathbb P_{P_e}(\widehat y=1)\right|
\le\mathbb P_{P_e}(y(X)\ne\widehat y)
\le\varepsilon_e.
\]
The variational definition of total variation and the triangle inequality give
\begin{align*}
\mathrm{TV}(Q_1,Q_0)
&\ge |Q_1(A_1)-Q_0(A_1)|\\
&\ge |p_1-p_0|-|p_1-Q_1(A_1)|-|p_0-Q_0(A_1)|,
\end{align*}
which proves \eqref{eq:approx-two-sample-bound}. Similarly, if
$Q^*(A_1)=\alpha/2$,
\[
\mathrm{TV}(Q_1,Q^*)
\ge |Q_1(A_1)-Q^*(A_1)|
\ge \left|p_1-\frac{\alpha}{2}\right|-\varepsilon_1.
\]
If the left-hand side is at most $\alpha/2-\varepsilon_1$, then
$|p_1-\alpha/2|\le\alpha/2$ and hence $p_1\le\alpha$. The final claim follows directly from \eqref{eq:approx-two-sample-bound} and the defining property of a characteristic kernel.
\end{proof}

\subsubsection{Class-Conditional Latent Drift}

Assumption~\ref{ass:conditional-invariance} asks the class-conditional latent laws to be identical in the two domains. For the trained encoder they move slightly even under benign shifts (Figure~\ref{fig:invariance_check}). The next proposition shows that such drift shifts the statistic by at most its own size, so small drift costs little.

Let $R_a\coloneqq\operatorname{Law}_{P_0}(U\mid y(X)=a)$ and, for a distribution $P$, let $p\coloneqq\mathbb P_P(y(X)=1)$, $Q_P\coloneqq P\circ\phi^{-1}$, and $S_{P,a}\coloneqq\operatorname{Law}_P(U\mid y(X)=a)$. If a class has probability zero, its conditional law may be chosen arbitrarily because its mixture weight is zero. Define
\[
\overline Q_p\coloneqq(1-p)R_0+pR_1,
\qquad
\delta_D(P)\coloneqq(1-p)D(S_{P,0},R_0)+pD(S_{P,1},R_1).
\]

\begin{proposition}[Robustness to class-conditional latent drift]
\label{prop:conditional-drift}
For any integral probability metric $D$,
\begin{equation}
\label{eq:conditional-drift-bound}
\left|D(Q_P,Q^*)-D(\overline Q_p,Q^*)\right|\le\delta_D(P).
\end{equation}
Consequently, if $\delta_D(P)\le\varepsilon_{\rm inv}$ for every $P\in\calP_0$, then
\begin{equation}
\label{eq:robust-null-boundary}
\sup_{P\in\calP_0}D(Q_P,Q^*)
\le
\sup_{0\le p\le\alpha}D(\overline Q_p,Q^*)+\varepsilon_{\rm inv}.
\end{equation}
Moreover, for any fixed $P_1\notin\calP_0$, let
$p_1\coloneqq\mathbb P_{P_1}(y(X)=1)$. Then
\begin{align}
&D(Q_{P_1},Q^*)-\sup_{P\in\calP_0}D(Q_P,Q^*) \nonumber\\
&\quad\ge
D(\overline Q_{p_1},Q^*)
-\sup_{0\le p\le\alpha}D(\overline Q_p,Q^*)
-\delta_D(P_1)-\varepsilon_{\rm inv}.
\label{eq:robust-power-margin}
\end{align}
Thus the population null boundary grows by at most $\varepsilon_{\rm inv}$. If $Q^*$ is the reference of the calibrated test and its threshold is inflated by $\varepsilon_{\rm inv}$, then, under the consistency conditions of Appendix~\ref{app:test-algo} and as $m$, $n_{\rm cal}$, and $B$ diverge and $h(\mathcal G_\alpha)\to0$, the rejection probability tends to zero at every null law with $\delta_D(P)\le\varepsilon_{\rm inv}$ that lies strictly inside the inflated boundary, and to one at every alternative for which the right-hand side of \eqref{eq:robust-power-margin} is positive. Assumption~\ref{ass:conditional-invariance} makes all these drift radii zero.
\end{proposition}

\begin{proof}
Since $Q_P=(1-p)S_{P,0}+pS_{P,1}$, convexity of an integral probability metric gives
\[
D(Q_P,\overline Q_p)
\le(1-p)D(S_{P,0},R_0)+pD(S_{P,1},R_1)
=\delta_D(P).
\]
The reverse triangle inequality proves \eqref{eq:conditional-drift-bound}. If $P\in\calP_0$, then $p\le\alpha$, and \eqref{eq:robust-null-boundary} follows by taking the supremum. Applying the lower side of \eqref{eq:conditional-drift-bound} to $P_1$ and subtracting \eqref{eq:robust-null-boundary} yields \eqref{eq:robust-power-margin}.
\end{proof}

\paragraph{Per-law guarantee with the centered reference.} For the centered reference $Q^*=Q_{\alpha/2}\coloneqq(1-\alpha/2)R_0+(\alpha/2)R_1$ of the calibrated test, the homogeneity argument behind \eqref{eq:ipm-mixture-identity}, which holds for any integral probability metric, gives $D(\overline Q_p,Q^*)=|p-\alpha/2|\,D(R_0,R_1)$, and the population null radius is $r_\alpha=(\alpha/2)D(R_0,R_1)$. By \eqref{eq:conditional-drift-bound}, the population decision for a law $P$, which rejects when $D(Q_P,Q^*)>r_\alpha$, is therefore correct whenever
\[
\delta_D(P)<\min\{p,\ \alpha-p\}\,D(R_0,R_1)\ \ \text{if } p\le\alpha,
\qquad
\delta_D(P)<(p-\alpha)\,D(R_0,R_1)\ \ \text{if } p>\alpha.
\]
Drift that moves latent codes toward the other class, along the segment between $R_0$ and $R_1$, attains these bounds for alternatives and for null laws with $p\ge\alpha/2$. For null laws with $p<\alpha/2$, such drift first moves the latent law toward the reference; there the binding case is drift away from the segment, which can raise a false alarm once $\delta_D(P)>p\,D(R_0,R_1)$.

\paragraph{Illustration.}
Figure~\ref{fig:assumption}(c) illustrates Proposition~\ref{prop:conditional-drift} on the trained \texttt{L2T} latent of the mass--spring--damper system. Its horizontal axis is the safety margin $p_1-\alpha$, with $p_1=\mathbb E_{X\sim P_1}[y(X)]$, and its vertical axis, the violation level of Assumption~\ref{ass:conditional-invariance}, is the relative drift $\delta_D(P_1)/D(R_0,R_1)$. We inject class-conditional drift of known size toward the other class: for $p_1\le\alpha$, a fraction of the stable class receives violation-like latent codes, and for $p_1>\alpha$, a fraction of the violating class receives stable-like codes. Under this drift the population decision is correct exactly when $|p_1-\alpha|\,D(R_0,R_1)>\delta_D(P_1)$, and by the per-law guarantee above this drift is the least favorable one for alternatives and for null laws with $p_1\ge\alpha/2$. With $m=2{,}000$ deployment contexts, and with the reference and the calibration batches of Algorithm~\ref{alg:l2t_pipeline_composite} taken from the two halves of a labeled split of $20{,}000$ contexts, the test decides correctly with high probability below this boundary and fails above it; the width of the transition reflects the finite deployment sample. Each cell of Figure~\ref{fig:assumption}(c) averages $200$ deployment batches for each of five independent calibration splits, and hatched cells would require more drift than the violating class can carry. In words, a class-conditional drift of size $\delta_D(P_1)$ shifts the violation probability seen by the test by at most $\delta_D(P_1)/D(R_0,R_1)$. Under drift toward the other class, deployments far from the tolerance are therefore robust to sizeable violations of Assumption~\ref{ass:conditional-invariance} and only near-boundary deployments are sensitive to it; in every case, a representation that separates the two classes more widely tolerates more drift, because all the bounds above scale with $D(R_0,R_1)$.

\subsection{Proof of Theorem~\ref{thm:testing_performance}}
\label{app:testing-performance}

\paragraph{Full statement.} Theorem~\ref{thm:testing_performance} conditions on the trained encoder $\widehat\phi$, whose class-conditional latent laws are $R_0$ and $R_1$, and applies to a deployment law $P_1$ that satisfies Assumption~\ref{ass:conditional-invariance}, with $p_1\coloneqq\mathbb P_{P_1}(y(X)=1)$. Type-I error control holds for every fixed deployment size $m$ as $n_{\rm cal},B\to\infty$ and $h(\mathcal G_\alpha)\to0$, under conditions ($i$)--($iv$) of Appendix~\ref{app:test-algo} and continuity of the distribution function of the oracle statistic $\widehat T_m(Q_{p_1})$, defined below, at the oracle threshold $r_{\alpha,\beta}^{(m)}$. Power holds as $m$, $n_{\rm cal}$, and $B$ diverge and $h(\mathcal G_\alpha)\to0$, under the consistency conditions of Appendix~\ref{app:test-algo} and $D(R_0,R_1)>0$.

\paragraph{Intuition.} Validity holds because the oracle threshold is the largest null quantile over the whole segment $\lambda\in[0,\alpha]$, so it controls the Type-I error at every null law, and the empirical threshold converges to it. Power holds because every alternative lies a distance $(p_1-\alpha)D(R_0,R_1)$ beyond the null radius, so the statistic eventually exceeds the threshold.

\begin{proof}
For $\lambda\in[0,\alpha]$, write
$Q_\lambda\coloneqq(1-\lambda)R_0+\lambda R_1$, so the population centered reference is $Q_{\alpha/2}$. For a fixed deployment sample size $m$, let $\widehat Q_{\lambda,m}$ be the empirical law of $m$ i.i.d. draws from $Q_\lambda$ and set $\widehat T_m(Q_\lambda)\coloneqq D(\widehat Q_{\lambda,m},Q_{\alpha/2})$. Define
\[
r_\beta^{(m)}(\lambda)
\coloneqq
\inf\left\{t:\mathbb P_{Q_\lambda}\!\left(\widehat T_m(Q_\lambda)\le t\right)
\ge1-\beta\right\}.
\]
The oracle composite-null threshold is
\begin{equation}
\label{eq:composite-critical-value}
r_{\alpha,\beta}^{(m)}
\coloneqq
\sup_{0\le \lambda\le\alpha}
r_\beta^{(m)}(\lambda).
\end{equation}
By the definition of the quantile, for every $\lambda\in[0,\alpha]$,
\[
\mathbb P_{Q_\lambda}\!\left(
\widehat T_m(Q_\lambda)>r_\beta^{(m)}(\lambda)
\right)\le\beta.
\]
Since $r_{\alpha,\beta}^{(m)}$ is at least each of these pointwise quantiles,
\[
\mathbb P_{Q_\lambda}\!\left(
\widehat T_m(Q_\lambda)>r_{\alpha,\beta}^{(m)}
\right)\le\beta
\]
for every $\lambda\in[0,\alpha]$. Under
Assumption~\ref{ass:conditional-invariance}, any $P_1\in\calP_0$ has
$p_1\in[0,\alpha]$ and induces
$Q_{P_1}=Q_{p_1}$, proving the oracle Type-I bound.

We next relate the empirical threshold to its oracle counterpart. Let $\widehat r_\beta^{(m)}(\lambda)$ be the Monte Carlo quantile of $D(\widehat Q_\lambda^{(b)},\widehat Q^*)$, where both $\widehat Q_\lambda^{(b)}$ and the fixed reference $\widehat Q^*$ are constructed from the same empirical class-conditional laws $\widehat R_0,\widehat R_1$. Then
\begin{align*}
\left|\widehat r_{\alpha,\beta}-r_{\alpha,\beta}^{(m)}\right|
&\le
\max_{\lambda\in\mathcal G_\alpha}
\left|\widehat r_\beta^{(m)}(\lambda)-r_\beta^{(m)}(\lambda)\right|\\
&\quad+
\left|
\max_{\lambda\in\mathcal G_\alpha}r_\beta^{(m)}(\lambda)
-\sup_{\lambda\in[0,\alpha]}r_\beta^{(m)}(\lambda)
\right|.
\end{align*}
The first term converges to zero in probability as the class-conditional calibration sample sizes and $B$ diverge: the assumptions in Appendix~\ref{app:test-algo} cover estimation of the reference as well as estimation of the simulated deployment laws and their quantiles. The second term converges to zero as $h(\mathcal G_\alpha)\to0$ by continuity in $\lambda$. Hence
\[
\widehat r_{\alpha,\beta}-r_{\alpha,\beta}^{(m)}\xrightarrow{p}0.
\]
For every fixed $m$ and $P_1\in\calP_0$, consistency of the shared calibration estimates and independence of calibration and deployment data imply that $\widehat T_{\rm test}$ converges in distribution to $\widehat T_m(Q_{p_1})$. Together with continuity of the distribution function of $\widehat T_m(Q_{p_1})$ at $r_{\alpha,\beta}^{(m)}$, this gives
\[
\limsup_{\substack{n_{\rm cal},B\to\infty\\
h(\mathcal G_\alpha)\to0}}
\mathbb P_{P_1}\!\left(
\widehat T_{\rm test}>\widehat r_{\alpha,\beta}
\right)
\le\beta.
\]

For consistency, fix $P_1\notin\calP_0$ and define the population separation
\[
\Delta
\coloneqq
D(Q_{P_1},Q_{\alpha/2})-
\sup_{0\le\lambda\le\alpha}
D(Q_\lambda,Q_{\alpha/2}),
\]
which is positive, as shown at the end of the proof.
The consistency assumptions in Appendix~\ref{app:test-algo} imply
\[
\widehat T_{\rm test}\xrightarrow{p}D(Q_{P_1},Q_{\alpha/2}),
\qquad
\widehat r_{\alpha,\beta}
\xrightarrow{p}
\sup_{0\le\lambda\le\alpha}D(Q_\lambda,Q_{\alpha/2}).
\]
Consequently,
\[
\widehat T_{\rm test}-\widehat r_{\alpha,\beta}
\xrightarrow{p}\Delta>0,
\]
which proves that the rejection probability converges to one.

Finally, let $C_D\coloneqq D(R_0,R_1)$, which is positive by assumption, for example when $D$ is a separating IPM and $R_0\neq R_1$. The same mixture identity used in Theorem~\ref{thm:reformulation} gives
\[
D(Q_\lambda,Q_{\alpha/2})
=\left|\lambda-\frac{\alpha}{2}\right|C_D.
\]
Therefore,
\[
\sup_{0\le\lambda\le\alpha}D(Q_\lambda,Q_{\alpha/2})
=\frac{\alpha}{2}C_D,
\]
whereas every alternative has
$p_1>\alpha$ and hence
\[
D(Q_{P_1},Q_{\alpha/2})-
\sup_{0\le\lambda\le\alpha}D(Q_\lambda,Q_{\alpha/2})
=(p_1-\alpha)C_D>0.
\]
Thus every fixed alternative is positively separated from the composite null by the centered population reference whenever $C_D>0$. The same ordering holds if squared MMD is used as the statistic.
\end{proof}

\section{Details of Experimental Configuration and Additional Results}
\label{app:exp-conf-results}

This section gives the configurations of the simulation experiments (Appendix~\ref{app:synthetic}) and of the TSA benchmark (Appendix~\ref{app:real-data}), together with additional results. For the simulation experiments, the additional analyses ask, in turn, how the weights of the training objective shape the representation, how \texttt{L2T} compares with simulating the deployment trajectories or learning a trajectory surrogate, and whether Assumption~\ref{ass:conditional-invariance} holds for the trained encoder.

\subsection{Simulation Experiments}
\label{app:synthetic}

This section describes the synthetic experimental setups used to evaluate our method. We consider two representative systems: an unconstrained mass--spring--damper system and a constrained Cartesian pendulum governed by DAEs. The spring is an ordinary differential equation with a two-dimensional context, whose stability map can be evaluated by simulation over the whole context space; the pendulum adds an algebraic constraint and tests the method on a DAE.

\paragraph{Mass--Spring--Damper System.}
We consider a one-dimensional mass--spring--damper system with state $s(t)\in\R$ evolving as
\[
m \ddot{s}(t) + c \dot{s}(t) + k s(t) = u(t;X), \quad t\in[0,T],
\]
where $\theta=(m,c,k)=(1,1,2)$ are fixed physical parameters. Unless otherwise specified, we initialize the system at rest with
$s(0)=0$ and $\dot s(0)=0$. The external forcing is parameterized by a random context $X=(X_1,X_2)\in\scrX$ via
\[
u(t;X)=X_1 \sin(X_2 t),
\]
where $X_1$ controls the forcing magnitude and $X_2$ the oscillation frequency. For a trajectory $s_\theta(\cdot;X)$, we define the stability functional
\[
\eta\big(s_\theta(\cdot;X)\big)=\max_{t\in[t_\ell,t_u]} |s_\theta(t;X)|,
\]
over the window $[t_\ell,t_u]=[6,10]$, and declare the trajectory stable if $\eta\le\tau$ for a given tolerance $\tau>0$.

The baseline context distribution is $P_0=\mathcal N(\mathbf 0,I_2)$ over $(X_1,X_2)$. We consider two types of distribution shifts:
($i$) \emph{mean shifts}, $P_1(\delta)=\mathcal N(\mu_0+\delta v, I_2)$ with $\|v\|_2=1$, and
($ii$) \emph{covariance shifts}, $P_1(\delta)=\mathcal N(\mathbf 0,\Sigma_\delta)$ with
\[
\Sigma_\delta=\begin{pmatrix}1+\delta_1 & 0\\ 0 & 1+\delta_2\end{pmatrix}.
\]
These perturbations induce regimes corresponding to null, boundary-transition, and alternative settings (Table~\ref{tab:shift_settings}).

\paragraph{Cartesian Pendulum System.}
We consider a constrained dynamical system governed by DAEs. Let $s(t)\in\R^3$ denote the bob position and $u(t;X)\in\R^3$ the pivot trajectory. The dynamics are
\begin{subequations}
\begin{align}
m\ddot s(t) &= m g_0 + (s(t)-u(t;X))\,a(t), \quad g_0=(0,0,-g)^\top,\\
0 &= \tfrac12\big(\|s(t)-u(t;X)\|_2^2-\ell^2\big),
\end{align}
\end{subequations}
where $\theta=(m,\ell,g)=(1,1,9.81)$ and $a(t)$ enforces the rigid rod constraint. Unless otherwise specified, we initialize the system with bob position
\[
s(0) = (0.15,\,0,\,-0.9887)^\top,
\]
initial velocity
\[
\dot s(0) = (0,\,0.8,\,0)^\top,
\]
and initial pivot position
\[
u(0;X) = (0,\,0,\,0)^\top.
\]

The pivot trajectory is parameterized by a random context $X=(X_1,X_2,X_3)\in\scrX$ via
\[
u(t;X)=X_1 \exp\!\left(-\frac{(t-t_0)^2}{2X_3^2}\right)\sin\big(X_2 (t-t_0)\big)\,q,
\]
where $q=(1,0,0)^\top$ and the excitation is centered at $t_0=T/2=5$. Here $X_1$, $X_2$, and $X_3$ control the excitation magnitude, oscillation frequency, and temporal spread, respectively.

We define stability via the maximum swing angle
\[
\eta(s_\theta(\cdot;X))=\max_{t\in[t_0,T]} \varphi_\theta(t;X),
\quad
\varphi_\theta(t;X)=\arccos\!\left(
-\frac{(s_\theta(t;X)-u(t;X))^\top e_z}{\ell}
\right),
\]
with $e_z=(0,0,1)^\top$. A trajectory is stable if $\eta\le\tau$ for $\tau\in(0,\pi/2)$. Geometrically, this restricts motion on the sphere of radius $\ell$ to a downward cone, equivalently a bounded region in the horizontal projection.

For the baseline distribution, we fix $X_3=1$ and sample $(X_1,X_2)$ independently using scaled Beta distributions:
\[
X_1 \sim A_{\min} + (A_{\max}-A_{\min})\,\mathrm{Beta}(\alpha_1,\beta_1), \quad [A_{\min},A_{\max}]=[0.01,0.3],
\]
\[
X_2 \sim \omega_{\min} + (\omega_{\max}-\omega_{\min})\,\mathrm{Beta}(\alpha_2,\beta_2), \quad [\omega_{\min},\omega_{\max}]=[1,5],
\]
with $\alpha_1,\beta_1=1,2$ and $\alpha_2,\beta_2=1,4$.
We consider two classes of shifts:
($i$) \emph{marginal shifts}, induced by changing the Beta parameters of $(X_1,X_2)$, and
($ii$) \emph{dependence shifts}, induced by introducing correlation $\rho$ between $X_1$ and $X_2$.
These shifts define null, boundary-transition, and alternative regimes (Table~\ref{tab:shift_settings}).

\begin{table}[!t]
\centering
\caption{Summary of the synthetic shift settings.}
\label{tab:shift_settings}
\footnotesize
\renewcommand{\arraystretch}{1.1}
\resizebox{\linewidth}{!}{
\begin{tabular}{ccccc}
\toprule
\textbf{Example} & \textbf{Shift} & \textbf{Null} & \textbf{Boundary-transition} & \textbf{Alternative} \\
\midrule
\multirow{2}{*}{Spring}
& Mean
& \makecell{$\mu_0=(0,0)$,\\ $v=(0,1)$}
& \makecell{$\mu_0=(0,0)$,\\ $v=(1,0)$}
& \makecell{$\mu_0=(-1,0)$,\\ $v=(0,1)$} \\[1em]
& Covariance
& \makecell{$\Sigma_\delta=
\begin{pmatrix}
1 & 0\\
0 & 1+\delta
\end{pmatrix}$}
& \makecell{$\Sigma_\delta=
\begin{pmatrix}
1+\delta & 0\\
0 & 1
\end{pmatrix}$}
& \makecell{N/A} \\
\midrule
\multirow{2}{*}{Pendulum}
& Marginal
& \makecell{$(\alpha_1,\beta_1)=\left(1+\tfrac{10}{3}\delta,\ 2+\tfrac{20}{3}\delta\right)$\\
$(\alpha_2,\beta_2)=(1,4)$}
& \makecell{$(\alpha_1,\beta_1)=(1,2)$\\
$\alpha_2=1+1.5\delta,\ \beta_2=4-1.5\delta$}
& \makecell{$(\alpha_1,\beta_1)=(1.5,1.5)$\\
$\alpha_2=1+1.5\delta,\ \beta_2=4-1.5\delta$} \\[1em]
& Dependence
& \makecell{$\rho=-\delta$}
& \makecell{$\rho=\delta$}
& \makecell{N/A} \\
\bottomrule
\end{tabular}
}
\end{table}

\paragraph{Experimental Configurations.}
Across both synthetic examples, trajectories are simulated over the time interval $[0,T]$ with $T=10$ and discretized using a uniform step size $\Delta t=0.05$, yielding $201$ time points per trajectory. For each setting in Table~\ref{tab:shift_settings}, contexts are drawn from the baseline distribution $P_0$ for training and held-out evaluation, and from the corresponding shifted distribution $P_1$ for testing. We use $5{,}000$ trajectories for training, $1{,}000$ held-out baseline trajectories for evaluation, and $2{,}000$ test contexts per run under the shifted regime. Unless otherwise specified, the training data are constructed to contain approximately $90\%$ stable and $10\%$ unstable trajectories. The stability threshold is set to $\tau=1$ for the spring system and to $\tau=35^\circ$ (equivalently, $\tau \approx 0.61$ radians) for the pendulum. Throughout the synthetic experiments, we use a one-dimensional latent representation, \ie, $d=1$. All rejection rates in Figures~\ref{fig:spring_result} and~\ref{fig:pendulum_result} and in Table~\ref{tab:result_summary} are computed over $100$ independent runs.

Models are trained for $100$ epochs using Adam with learning rate $3\times 10^{-3}$, regularization weight $\lambda=0.01$, and stability-aware weighting parameter $\omega=2$.
For the binary classifier, we use a multi-layer network trained with the same Adam configuration and logistic binary cross-entropy loss. All main neural-network experiments are run on a workstation with an NVIDIA GeForce RTX 5070 Ti GPU with 16GB memory. Lightweight preprocessing, plotting, and small-scale sanity checks are also performed on a Mac with an Apple M4 chip.

\paragraph{Sensitivity Analysis and Ablation Studies.}

We study how the weights of the three loss terms shape the learned representation. Composite-null calibration controls Type-I error when the class-conditional latent laws are invariant (Assumption~\ref{ass:conditional-invariance} and Proposition~\ref{prop:conditional-drift}), so the loss weights matter through the representation itself: whether the latent stays invariant under benign shifts and how strongly it responds to risk-increasing ones. We measure both with the uncalibrated two-sample test between the baseline and shifted latent samples. Its rejection rate under a null shift, where the violation rate is unchanged, measures how much the latent moves under a shift that should leave it invariant; its Type-II error under a boundary-transition shift measures how weakly the latent responds to a risk-increasing shift. Results are averaged over multiple random seeds.

We first study the contribution of each term in the training objective. Recall that the stability-aware reconstruction loss can be decomposed into a standard reconstruction term and an asymmetric boundary penalty. This allows us to isolate the effect of the asymmetric stability-aware component and write the total loss as
\begin{equation}
\label{eq:simplex_loss}
\mathcal{L}_{\mathrm{total}}
=
\lambda_{\mathrm{rec}} \mathcal{L}_{\mathrm{rec}}
+
\lambda_{\mathrm{reg}} \mathcal{L}_{\mathrm{reg}}
+
\lambda_{\mathrm{stb}} \mathcal{L}_{\mathrm{stb}},
\end{equation}
where
$\lambda_{\mathrm{rec}}+\lambda_{\mathrm{reg}}+\lambda_{\mathrm{stb}}=1$.
Here, $\mathcal{L}_{\mathrm{rec}}$ is the standard unweighted mean squared reconstruction loss, $\mathcal{L}_{\mathrm{reg}}$ is the uniformity regularization term, and $\mathcal{L}_{\mathrm{stb}}$ isolates the asymmetric penalty applied to underestimated unstable trajectories:
\begin{equation}
\label{eq:loss_stb_isolated}
\mathcal{L}_{\mathrm{stb}}
\coloneqq
\frac{1}{n}\sum_{i=1}^n
\frac{1}{K_i}\sum_{k=1}^{K_i}
\big(\omega_{\mathrm{stb}}(s_{i,k},\hat{s}_{i,k})-1\big)
\|s_{i,k}-\hat{s}_{i,k}\|^2 .
\end{equation}

Because $\mathcal{L}_{\mathrm{reg}}$ and $\mathcal{L}_{\mathrm{stb}}$ have different numerical scales from the reconstruction loss, sweeping the full simplex is not informative in practice: assigning scale-one weights to the auxiliary terms often causes poor testing performance. We therefore examine a local reconstruction-dominant simplex with vertices
\[
(\lambda_{\mathrm{rec}},\lambda_{\mathrm{reg}},\lambda_{\mathrm{stb}})
\in
\{(1,0,0),\,(0.9,0.1,0),\,(0.9,0,0.1)\}.
\]
This region varies $\lambda_{\mathrm{reg}}$ and $\lambda_{\mathrm{stb}}$ within a practically meaningful range while maintaining a dominant weight on trajectory reconstruction.

Figure~\ref{fig:sensitivity_simplex} shows the results at significance level $\beta=0.1$. Without a shift (left panel), the rejection rate stays between $0$ and $0.15$ for every weight setting, close to the nominal level. Under the null shift (middle panel), settings with a larger uniform-regularization weight yield a latent that moves with the benign shift and is rejected almost always, whereas a band of settings that combine reconstruction with a small amount of uniform regularization and the stability-aware term leaves the latent nearly unchanged. The same band keeps a low Type-II error under the boundary-transition shift (right panel), while reconstruction alone or heavier uniform regularization does worse on at least one of the two criteria. This supports the design of \texttt{L2T}: reconstruction retains trajectory information, the stability-aware term keeps the latent responsive to risk-increasing shifts, and a small amount of uniform regularization spreads the latent without breaking its invariance under benign shifts.

To visualize the learned latent geometry, Figure~\ref{fig:latent_space_visual} plots the empirical latent density relative to the uniform reference as the shift grows. Under the null shift, the \texttt{L2T} density stays close to uniform across all shift magnitudes, whereas the uniform-only latent tilts away from it, which explains the false alarms of the uniform-only variant. Under the boundary-transition and alternative shifts, both latents depart increasingly from uniformity as $\delta$ grows.

\begin{figure}[!t]
    \centering
    \includegraphics[width=0.95\linewidth]{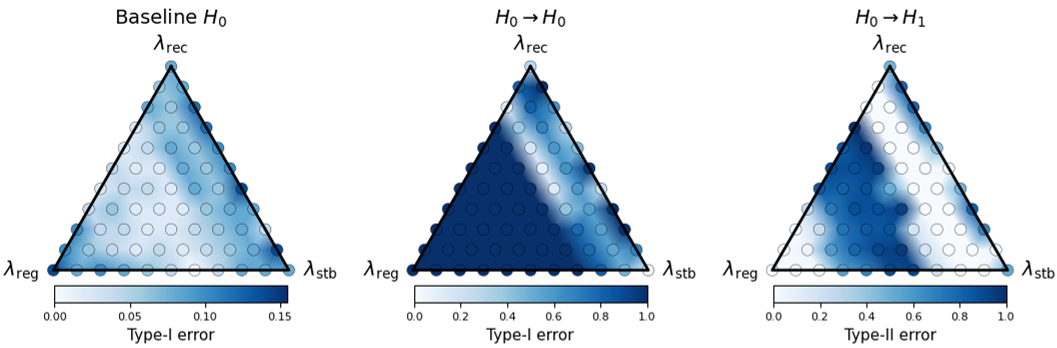}
    \caption{
    Sensitivity of the learned representation to the loss weights over a reconstruction-dominant simplex parameterized by $(\lambda_{\mathrm{rec}},\lambda_{\mathrm{reg}},\lambda_{\mathrm{stb}})$, evaluated with the uncalibrated two-sample test between baseline and shifted latent samples. The three panels correspond to the baseline without shift, a null shift, and a boundary-transition shift. Under the null shift a lower rejection rate indicates a latent that stays invariant to the benign shift; under the boundary-transition shift a lower Type-II error indicates a latent that responds to the risk-increasing shift.
    }
    \label{fig:sensitivity_simplex}
\end{figure}

\begin{figure}[!t]
    \centering
\includegraphics[width=0.95\linewidth]{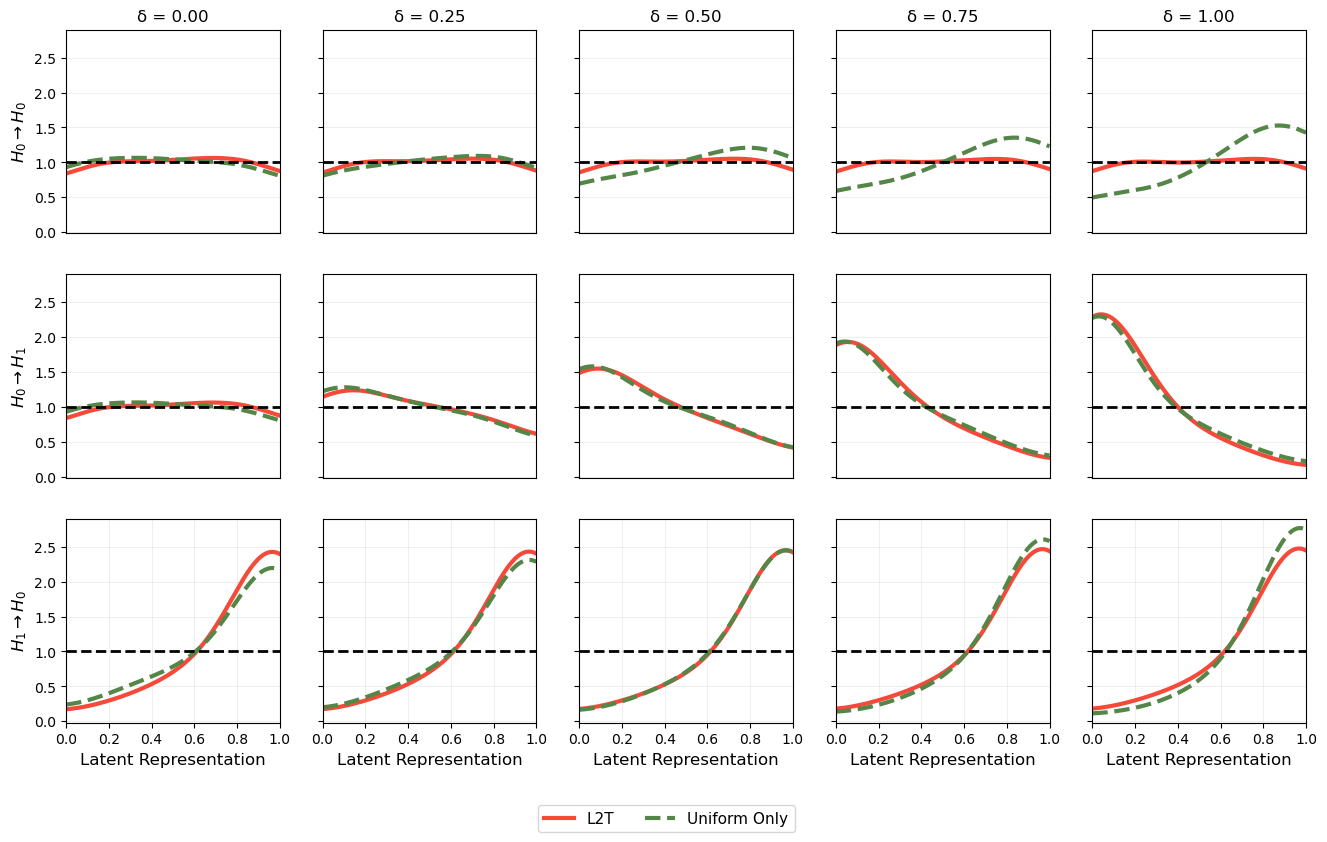}
\caption{
Latent-space under increasing shift magnitude $\delta$. Columns correspond to values of $\delta$, and rows show the latent density relative to the uniform reference under the null, boundary-transition, and alternative shifts. Under the null shift, the \texttt{L2T} latent stays near uniform across shift levels, whereas the uniform-only latent tilts away from it, which leads to false alarms. Under the boundary-transition and alternative shifts, both latents deviate increasingly from uniformity as $\delta$ grows.
}
\label{fig:latent_space_visual}
\end{figure}

\paragraph{Additional Baselines: Oracle Simulation and Trajectory Surrogate.}
We further consider two baselines that label deployment trajectories and apply a one-sided binomial test of $H_0:\mathbb P_{P_1}(y(X)=1)\le\alpha$ to the labels: \emph{oracle simulation}, which simulates the trajectories of $k$ deployment contexts, and \emph{trajectory surrogate + label}, which predicts them with a network trained on the baseline trajectories. Figure~\ref{fig:oracle_simulation_size} reports both on the mass--spring--damper system, in the settings of Figure~\ref{fig:spring_result}. Oracle simulation is valid for every simulation size $k$, and its power near the boundary grows with $k$, so the simulation budget matters for power. The trajectory surrogate is also valid but has low power in the boundary-transition setting, because it under-predicts violations. \texttt{L2T} needs no deployment simulation; near the boundary it rejects somewhat earlier than full simulation.

\begin{figure}[!t]
    \centering
    \includegraphics[width=\linewidth]{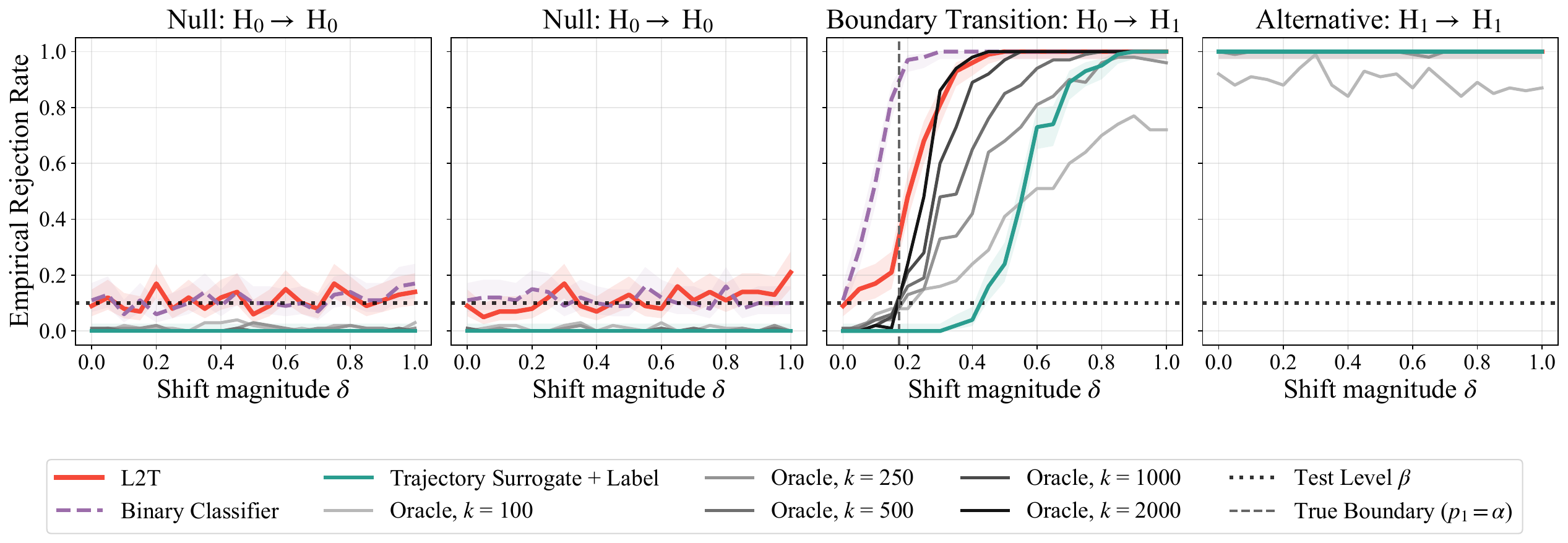}
    \caption{Additional baselines on the mass--spring--damper system; the panels are columns 1, 2, 4 and 5 of Figure~\ref{fig:spring_result}. \emph{Oracle simulation} simulates the first $k$ of the $2{,}000$ deployment contexts and applies an exact binomial test at $\alpha$; \emph{trajectory surrogate + label} applies the same test to the labels of trajectories predicted by a $2\times32$ MLP. The dashed vertical line marks where the true violation rate crosses $\alpha$, and the dotted line is the test level $0.1$.}
    \label{fig:oracle_simulation_size}
\end{figure}

\paragraph{Empirical Check of Assumption~\ref{ass:conditional-invariance}.}
Assumption~\ref{ass:conditional-invariance} cannot be checked at deployment, where trajectories are unobserved, but it can be checked in simulation, where the true labels are available. Figure~\ref{fig:invariance_check} measures how far the class-conditional latent laws of the trained \texttt{L2T} encoder drift from their baseline values $R_0$ and $R_1$ along the two null shifts of Figure~\ref{fig:spring_result}, in units of the null radius $r_\alpha$, the scale that matters in Proposition~\ref{prop:conditional-drift}. Although the contexts within each class move, their latent laws barely change: the drift stays below $0.8\,r_\alpha$ on average, close to the sampling floor, so Proposition~\ref{prop:conditional-drift} covers these shifts. The assumption does not hold for every covariate shift, however: along the risk-increasing mean shift of $X_1$ (third column of Figure~\ref{fig:spring_result}), the drift of the violation class reaches $9$--$20\,r_\alpha$, and the early rejections of \texttt{L2T} in that panel reflect this drift rather than a change in the violation rate.

\begin{figure}[!t]
    \centering
    \includegraphics[width=\linewidth]{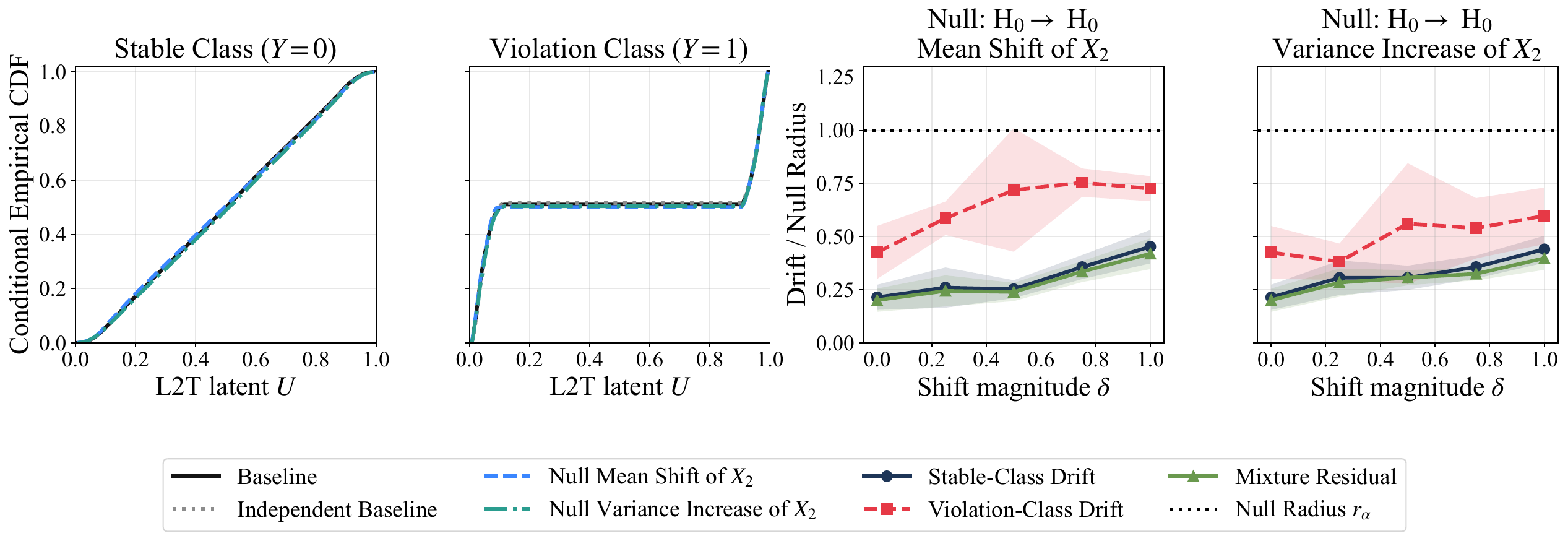}
    \caption{Empirical check of Assumption~\ref{ass:conditional-invariance} for the trained \texttt{L2T} encoder on the mass--spring--damper system. Left two panels: class-conditional empirical CDFs of the latent for the baseline, an independent baseline sample, and the two null shifts of Figure~\ref{fig:spring_result} at $\delta=1$. Right two panels: class-conditional drifts and mixture residual (MMD), divided by the null radius $r_\alpha=(\alpha/2)D(R_0,R_1)$, along the mean shift and the variance increase of $X_2$ (columns 1 and 2 of Figure~\ref{fig:spring_result}), where the violation rate stays at $8.0$--$8.1\%$ and $7.6$--$8.0\%$; $\delta=0$ is the independent-baseline sampling floor.}
    \label{fig:invariance_check}
\end{figure}

\subsection{Transient Stability Assessment on IEEE 39-Bus System}
\label{app:real-data}

We evaluate our framework on a transient stability assessment (TSA) benchmark generated from the IEEE New England 39-bus power system \citep{Sokolovic2024TSA}.
The dataset contains $12{,}852$ time-domain simulation trajectories with $50$ features, produced using DIgSILENT PowerFactory \citep{DIgSILENT} with Python-based automation. The benchmark complements the simulation experiments in two ways: the context is $50$-dimensional, and the evaluation varies the class proportions of the test set, a label shift under which Assumption~\ref{ass:conditional-invariance} holds by construction.

To cover diverse operating conditions, generation and load levels are varied from $80\%$ to $120\%$ in $5\%$ increments.
For each operating point, three-phase short-circuit faults are applied at seven locations ($0\%, 10\%, 20\%, 50\%, 80\%, 90\%, 100\%$) along transmission lines, with clearing times ranging from $0.1$s to $0.3$s.
Each simulation runs for $10$s, and post-fault trajectories are sampled at $0.01$s resolution from the fault clearance time to $0.6$s afterward.

Each trajectory records generator-level signals for $10$ generators, including active power, terminal voltage, excitation current, rotor speed, and rotor angle (relative to G02, the reference machine).
The dataset encodes stable trajectories as $1$ and unstable trajectories as $0$; this raw-data convention is the complement of the violation indicator $y$ in Definition~\ref{def:stab}. Approximately $42\%$ of the cases are unstable.

\begin{figure*}[!t]
    \centering
    \includegraphics[width=\linewidth]{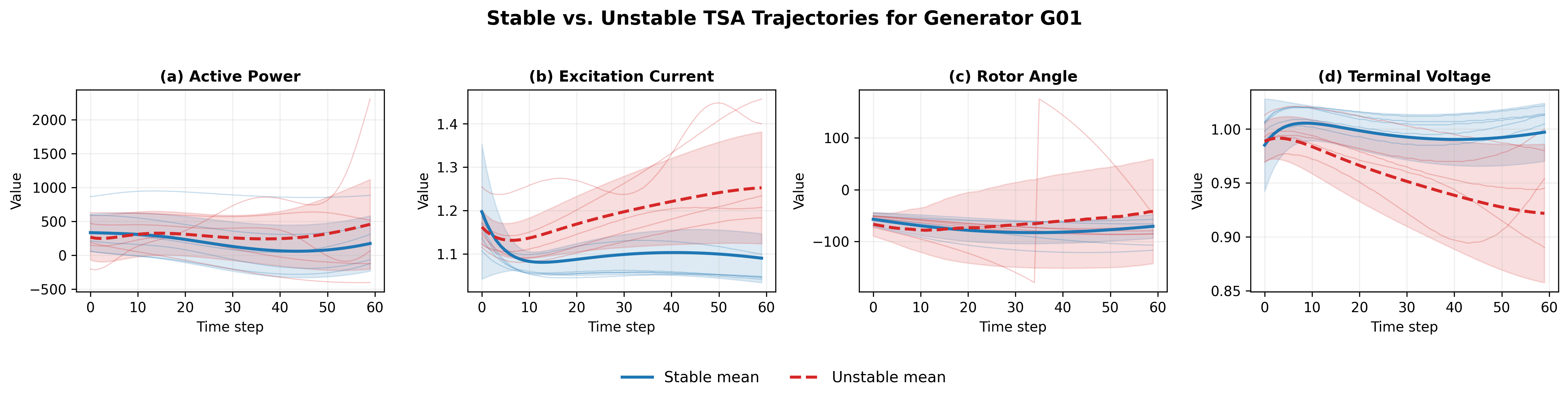}
    \caption{Trajectory visualization for generator G01 across four system variables: (a) active power, (b) excitation current, (c) rotor angle, and (d) terminal voltage. Blue curves denote stable trajectories and red curves denote unstable ones. Thin lines correspond to individual sampled trajectories, while solid lines represent the mean trajectories and shaded regions indicate one standard deviation.
    Stable trajectories remain tightly concentrated and exhibit smooth, consistent dynamics, whereas unstable trajectories show larger dispersion and systematic deviations, particularly in rotor angle and terminal voltage. These differences highlight the distributional characteristics that enable detection of instability.}
    \label{fig:tsa_traj_improved}
\end{figure*}

\begin{figure*}[t]
    \centering
    \includegraphics[width=1\textwidth]{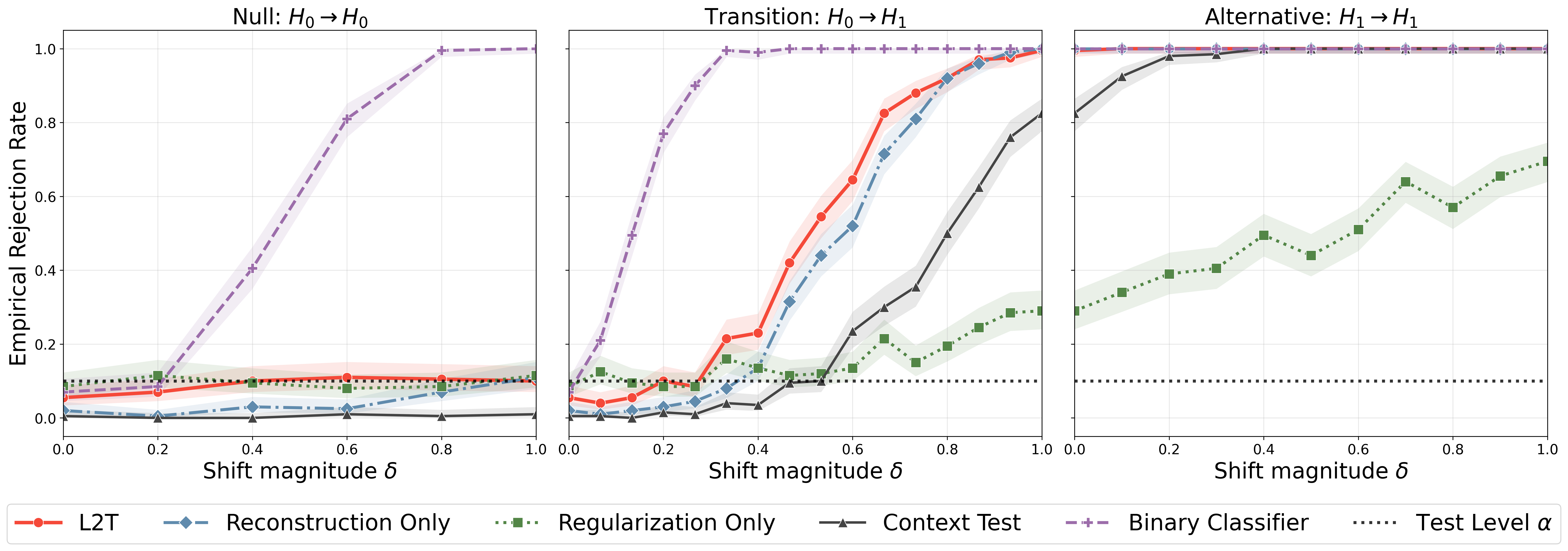}
    \caption{TSA rejection curves under three distribution-shift regimes. The null setting is defined by the reference mixture with unstable ratio $0.05$. Left: null shift from $H_0$ to $H_0$, where the unstable ratio decreases from $0.05$ to $0.00$. Middle: transition shift from $H_0$ to $H_1$, where the unstable ratio increases from $0.05$ to $0.20$. Right: alternative-regime shift from $H_1$ to $H_1$, where the unstable ratio varies from $0.20$ to $0.30$. The horizontal axis shows the normalized shift magnitude $\delta \in [0,1]$ for each regime. The empirical rejection rate of L2T remains close to the nominal test level under the null shift, increases smoothly through the transition regime, and is close to one in the alternative regime, indicating a favorable balance between Type-I error control and detection power.}
    \label{fig:tsa_delta_shift}
\end{figure*}

\begin{table}[!t]
\centering
\small
\caption{TSA results on Type-I and Type-II errors.
}
\label{tab:tsa_results_updated}
\begin{tabular}{lcc}
\toprule
Method & Type-I @ 0.00 & Type-II @ 0.20 \\
\midrule
L2T (flow, $d_v=8$)
& $0.100 \pm 0.035$
& $0.005 \pm 0.008$ \\

L2T-MMD ($d_v=1$)
& $0.145 \pm 0.041$
& $0.435 \pm 0.058$ \\

Reconstruction Only
& $0.110 \pm 0.036$
& $0.000 \pm 0.000$ \\

Regularization Only
& $0.115 \pm 0.037$
& $0.710 \pm 0.053$ \\

Context Test
& $0.010 \pm 0.011$
& $0.175 \pm 0.044$ \\

Binary Classifier
& $1.000 \pm 0.000$
& $0.000 \pm 0.000$ \\
\bottomrule
\end{tabular}
\end{table}

\paragraph{Data Split.}
We adopt a stable-versus-unstable testing protocol. Among $12{,}852$ trajectories, $7{,}460$ are stable and $5{,}392$ are unstable.
For training, we sample $5{,}000$ trajectories consisting of $95\%$ stable and $5\%$ unstable samples, resulting in $4{,}750$ stable and $250$ unstable trajectories.
The remaining stable and unstable samples are reserved for evaluation.

At test time, the nominal reference distribution is defined as a $95{:}5$ stable-to-unstable mixture, which serves as a representative point within the null family that includes nearby mixtures such as the fully stable setting. We then evaluate controlled shifts by varying the unstable component of the test set, including a downward null-preserving shift, a transition shift from the nominal mixture toward a higher-unstable mixture, and an alternative-regime sweep.

All features are normalized using statistics computed from the training set only, across both sample and time dimensions.
The same normalization is applied to all evaluation data to avoid leakage.

\paragraph{Trajectory Characteristics.}
To provide intuition on the TSA dataset, we visualize a subset of trajectories from generator G01 across four representative variables: active power, excitation current, rotor angle, and terminal voltage. For each variable, we randomly sample 5 stable and 5 unstable trajectories.
As shown in Figure~\ref{fig:tsa_traj_improved}, stable trajectories exhibit relatively consistent patterns in both magnitude and temporal dynamics, forming a concentrated and structured region in trajectory space. In contrast, unstable trajectories display significantly higher variability, with diverse shapes, amplitudes, and temporal behaviors.

\paragraph{Context Construction.}
Each trajectory has length $60$ with $50$ features per time step.
In our setting, the exogenous context is defined as the first observation of each trajectory:
\[
X_i = s_{i,1} \in \mathbb{R}^{50}.
\]
This corresponds to using the initial system state as the context variable. Since trajectories in this dataset begin immediately after fault clearance, this observation represents the post-fault system state at the onset of transient dynamics.
The remaining trajectory is used to construct a single training pair per sample.
Specifically, each trajectory is evenly split into two halves:
\[
o^{\mathrm{hist}}_i = (s_{i,1}, \dots, s_{i,T/2}), \quad
o^{\mathrm{next}}_i = (s_{i,T/2+1}, \dots, s_{i,T}),
\]
where $T=60$.
This yields one $(o^{\mathrm{hist}}, o^{\mathrm{next}})$ pair per trajectory.

\paragraph{Model Training.}
We train a conditional GRU model to predict future dynamics conditioned on the context and past trajectory segment.
For the main L2T-flow model, we use a latent dimension $d_v=8$ and add a normalizing-flow regularization term to the reconstruction-based training objective. We also report an L2T-MMD variant with $d_v=1$, which replaces the flow regularizer with the MMD-based regularization used in the original baseline.
All models are trained on GPU (CUDA).

\paragraph{Evaluation Protocol.}
At test time, each trajectory is mapped to a latent representation using only its context $X_i$. The reference set contains $950$ stable trajectories and $50$ unstable trajectories, corresponding to the nominal unstable ratio $0.05$. For each test ratio $r$, we keep the $950$ reference-stable trajectories fixed and sample
\[
n_{\mathrm{unstable}}(r) = \mathrm{round}\!\left(\frac{r}{1-r}\cdot 950\right)
\]
unstable trajectories from the held-out unstable pool. The resulting test batch therefore has unstable ratio approximately $r$, with the stable component fixed and the unstable component resampled across repeats.

For MMD-based methods, the rejection threshold is calibrated as the empirical $(1-\beta)$-quantile of the available null calibration statistics with $\beta=0.1$; this is the finite-calibration approximation of the composite critical value over the represented null settings. For the context baseline, we use the projected KS two-sample test on raw context vectors. For the Binary Classifier baseline, we compare the predicted-probability vectors of the reference and test batches using Welch's two-sample $t$-test. Each setting is repeated $200$ times, and we report the empirical rejection rate.

\paragraph{Additional TSA Results.}
Table~\ref{tab:tsa_results_updated} reports the Type-I error under the fully stable setting ($r=0.00$) and the Type-II error at the $80{:}20$ stable-to-unstable mixture ($r=0.20$). L2T-flow maintains Type-I error close to the target level while achieving near-zero Type-II error. Compared with L2T-MMD, the flow-based variant better preserves detection power on this high-dimensional TSA benchmark: L2T-MMD has noticeably higher Type-II error at $r=0.20$, while L2T-flow remains both calibrated and powerful. Some baselines achieve low Type-II error but suffer from over-rejection under the fully stable null, which is further reflected in the shift curves in Figure~\ref{fig:tsa_delta_shift}. On this high-dimensional benchmark, where Assumption~\ref{ass:conditional-invariance} holds, \texttt{L2T} with a sufficiently expressive latent is therefore both calibrated and powerful.

\end{document}